\documentclass[10pt]{amsart}
\usepackage{egastyle}

\usepackage{graphicx}
\usepackage{subcaption}
\usepackage{caption}
\usepackage{tikz,tikzcd}
\usepackage{verbatim}
\usepackage{adjustbox}
\usepackage{multicol,xcolor}

\usepackage{algorithm}
\usepackage{algpseudocode}
\usepackage{url}

\usepackage{xurl}

\newcommand{\wt}{\widetilde}

\newcommand{\wtg}{{\mathrm{wtg}}}
\newcommand{\stg}{{\mathrm{stg}}}
\newcommand{\gts}{{\mathrm{gts}}}

\usepackage{amssymb,amsmath,amscd,amsthm,amsfonts,wasysym,mathrsfs,amsxtra}
\input xy
\xyoption{all}

\begin{document}


\title{From data to concepts via  wiring diagrams}

\author{Jason Lo}
\address{Department of Mathematics, California State University, Northridge, USA}
\email{jason.lo@csun.edu}

\author{Mohammadnima  Jafari}
\address{Department of Computer Science, California State University, Northridge, USA}
\email{mohammadnima.jafari@csun.edu}

\date{\today}

\maketitle

\begin{abstract}
A wiring diagram is a labeled directed graph that represents an abstract concept such as a temporal process.  In this article, we introduce the notion of a quasi-skeleton wiring diagram graph, and prove that quasi-skeleton wiring diagram graphs correspond to Hasse diagrams.  Using this result, we designed algorithms  that extract wiring diagrams from sequential data.     We used our algorithms  in analyzing the behavior of an autonomous agent playing a computer game, and the algorithms correctly identified the winning strategies.  We compared the performance of our main algorithm with two other algorithms based on standard clustering techniques (DBSCAN and agglomerative hierarchical), including when some of the data was perturbed.  Overall, this article brings together techniques in category theory, graph theory, clustering,  reinforcement learning, and data engineering.
\end{abstract}

\tableofcontents

\begin{multicols}{2}

\section{Introduction}

\paragraph[Ologs and wiring diagrams] The notion of an ontology log, or \emph{olog} for short, was defined by Spivak in 2012 as a method for knowledge representation \cite{SpivakKent}.  Since ologs are based on category theory in mathematics, it means that all the tools and techniques from category theory are at one's disposal when using ologs.  In addition, since ologs are authored using words (in any written language of the author's choice), they are easily understandable to humans.  By design, an olog also represents a database schema in a natural way, meaning they  provide a framework for organizing data in an autonomous system.  Even though ologs are very similar to knowledge graphs in appearance, they are different in a fundamental way: arrows in an olog are always \emph{composable}.  This means, in practice, that the arrows in ologs often correspond to functions.


In 2024, the first author built on  the notion of an olog and considered the notion of a \emph{wiring diagram}, defined to be a  directed graph whose labels come from an olog \cite{LoMR1}.  Wiring diagrams, in the sense as defined in \cite{LoMR1}, make it easier than ologs to represent complex concepts that may involve `before-and-after' (e.g.\ temporal) relations among their components.  For example, if $p$ denotes a person, $s$ a coffee shop, and $c$ a cup of coffee, then the concept of `a person buying a cup of coffee' (or simply `buying coffee') can be represented by the wiring diagram

\begin{center}
\includegraphics[scale=0.6]{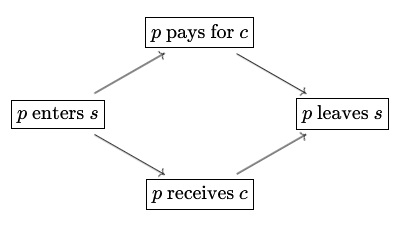}
\end{center}

This wiring diagram conveys the idea, that in order for the concept of `buying coffee' to occur, all these four events must occur:
\begin{itemize}
    \item A: $p$ enters a coffee shop;
    \item B: $p$ pays for the coffee;
    \item C: $p$ receives the coffee;
    \item D: $p$ leaves the coffee shop.
\end{itemize}
The arrows in the wiring diagram represent before-and-after relations, so the events $A, B, C, D$ are required to  satisfy these order relations:
\begin{itemize}
    \item $A$ must occur before $B$ as well as $C$;
    \item both $B$ and $C$ must occur before $D$.
\end{itemize}
Overall, the concept `buying coffee' is considered to have occurred if all the individual events $A, B, C, D$ have  occurred, and all the  order relations above are satisfied.

In practice, one can think of  a wiring diagram as a directed acyclic graph where the vertex labels  correspond to sensor readings   \cite{LoMR1}.  The sensor could be a simple, physical sensor such as a thermometer, in which case an associated label in a wiring diagram could be `the ambient temperature reaches 30 degrees Celsius'.  The sensor could also be a complex, non-physical sensor such as a software detecting anomalies in online user behaviors, in which case an associated label in a wiring diagram could be `customer $X$ makes a highly atypical financial transaction'.

Ologs themselves can already be used to perform basic deductive reasoning \cite{SpivakKent}.  Since wiring diagrams are extensions of ologs (in the sense that the label of each node in a wiring diagram must come from an olog), wiring diagrams can be used to represent more complicated types of reasoning, such as analogical reasoning and problem solving \cite{LoMR1}.

\paragraph[The problem] When it comes to using  wiring diagrams to perform reasoning within an autonomous system, two  fundamental problems arise:
\begin{itemize}
    \item Problem I: How does the autonomous system translate sensor data to wiring diagrams?
    \item Problem II: How does the autonomous system perform general reasoning by manipulating wiring diagrams?
\end{itemize}
Problem I is part of the broader  problem of, \emph{how does an autonomous system understand its environment by defining concepts on its own, based on data collected through its sensors?}  In our context, this means the following: suppose an autonomous system makes multiple observations of a human buying coffee from a coffee shop; how then does the autonomous system form the concept of `buying coffee' using the language of wiring diagrams?  In cognitive psychology, for example, Schank and Abelson suggested that children learn to grasp abstract concepts such as `eating at a restaurant' by experiencing it multiple times and then forming script-like structures associated to the concept \cite{schank1977scripts}.  In this paper, we attempt to realize this process by designing  algorithms that can be implemented in autonomous systems.


In designing such  algorithms, however, one faces the following  mathematical problems: What are the wiring diagrams  we are dealing with, and how can we find all of them?   That is, how do we characterize the wiring diagrams that are relevant to constructing our algorithms, and how do we enumerate all such wiring diagrams?  We answer  these questions in Theorem \ref{thm:ON-MR06p2-p15} by showing that the underlying graphs of the wiring diagrams that we work with - which we call \emph{quasi-skeleton wiring diagram graphs} - correspond to Hasse diagrams in graph theory, the enumeration problem of which is already well-known \cite{OEISposet}.

\paragraph[Our algorithms] Once we understood the space of wiring diagrams we are dealing with, we were able to design algorithms that extract wiring diagrams from data.  Our main algorithm, Algorithm \ref{algo:3-v1}, takes as input a collection of sequences representing multiple observations, and produces as output collections of matrices with entries in $\{0, 1\}$.  Each matrix corresponds to a wiring diagram, and each collection of matrices corresponds to a group of abstract concepts that are embedded within the input data.

Algorithm \ref{algo:3-v1} can be thought of as a type of clustering algorithm, since it attempts to aggregate data points into a smaller number of data points that are representative of the majority.  For us, each `data point' comes from  a single sequence of sensor data, which can be translated into a matrix associated to a wiring diagram via Algorithm \ref{algo:1-v4}.  The major difference between our  algorithm and standard clustering algorithms is, that while existing clustering algorithms require a measure of similarity (often a metric) between data points, our algorithm makes use of the poset structure among Hasse diagrams to ``generalize'' the concepts corresponding to individual data points to a smaller number of more abstract, but representative, concepts.  That is, our approach relies more on the \emph{intrinsic} structures of wiring diagrams rather than \emph{extrinsic} measures such as a metric, which does not necessarily ``see'' how some concepts are special cases, or generalizations of, other concepts.  Given this, we refer to  Algorithm \ref{algo:3-v1} as \emph{Hasse clustering}.

\paragraph[Implementation and testing]  To test our algorithms, we designed  two versions of a computer game where a player needs to collect certain items and use them in appropriate manners in order to win.  In version one, there is a unique winning strategy; in version two, there are two distinct winning strategies.  In each version, we used reinforcement learning (proximal policy optimization, or PPO) to train a computer agent to play the game, and   all the moves of the agent were recorded as it learned to play and eventually became capable of winning the game.   

In version one of the game,  Algorithm \ref{algo:3-v1}   produced the wiring diagram representing the unique winning strategy; in version two, the algorithm produced the two wiring diagrams representing the two possible winning strategies.  Our results demonstrate that autonomous systems can indeed be given the ability to extract abstract concepts from sensor data using the theory of wiring diagrams.



\paragraph[Outline of the paper]  In Section \ref{sec:WDgraphs}, we recall the definition of wiring diagrams from \cite{LoMR1} and introduce the notion of a quasi-skeleton wiring diagram graph.  In Section \ref{sec:WDvsHasse}, we prove our first main theorem (Theorem \ref{thm:ON-MR06p2-p15}), which says that quasi-skeleton wiring diagram graphs are precisely Hasse diagrams.  In Section \ref{sec:catQSWDG}, we set up the category-theoretic language necessary for comparing wiring diagrams and data in  Section  \ref{sec:consistWD}.  Section  \ref{sec:consistWD} culminates in  Theorem \ref{thm:ON-MR06p2-20-1}, which   can be seen as a universal property in the sense of category theory.  This theorem says that any comparison between  a wiring diagram and sequential data must go through  the `flattenings' of the wiring diagram (Definition \ref{def:flattening}).  

In Section \ref{sec:algos}, we lay out our algorithms for extracting wiring diagrams from data.  We then present our results from testing the algorithms in the context of player behavior in a computer game in Sections \ref{sec:application-1} and \ref{sec:application-2}.  In Section \ref{sec:comparison}, we include a brief comparison between our clustering algorithm (Hasse clustering) and two other approaches based on standard clustering algorithms  -  DBSCAN and (agglomerative) hierarchical clustering - in the context of our synthetic data.  In Section \ref{sec:comparison-corruption}, we compared the performance of Hasse clustering with DBSCAN- and hierarchical clustering-based approaches again, this time in a scenario where part of the data has been corrupted.  We conclude the article with  Section \ref{sec:article-conclusion}, which includes a brief discussion on future directions.



\section{Wiring diagram graphs} \label{sec:WDgraphs}



Recall that a \emph{directed graph} can be defined as a quadruple $G = (V, A, s, t)$ where $V$ is the set of vertices, $A$ is the set of arrows, and $s, t : A \to V$ are functions such that, for each arrow $a \in A$, $s(a)$ denotes the vertex where the arrow starts and $t(a)$ the vertex  where the arrow ends.

In \cite{LoMR1}, the first author defined wiring diagrams as a means to represent temporal processes.  Informally, a wiring diagram is a directed graph with labeled vertices and arrows, and where the arrows indicate before-and-after relations among the events represented by vertex labels.  These events are formally defined as  functions (called \emph{sensing functions}) taking on particular values; they are meant to represent sensors achieving particular readings.  Here, by a `sensor' we mean any entity - physical or non-physical - capable of a measurement (see \cite[Section 5]{LoMR1} for more details).

\begin{defn}[wiring diagram]\cite[Definition 5.4]{LoMR1}\label{def:WD}
    A \emph{wiring diagram} (\emph{WD}) is a quintuple 
    \[
    (V, A, s, t, \mathscr{L}_V)
    \]
    satisfying the following conditions.
    \begin{enumerate}
        \item[WD0.] $G=(V,A,s,t)$ is a finite directed graph, called the \emph{underlying graph} of the wiring diagram.  We will refer to elements of $V$ as \emph{vertices} or \emph{states}, and refer to elements of $A$ as \emph{arrows} or \emph{wires}. 
           
        \item[WD1.] $\mathscr{L}_V$ is an indexed set $\{ L_v\}_{v \in V}$ such that each $L_v$ is a  set of triples
        \[
          L_v = \{ (F_i, x_i, y_i) : 1 \leq i\leq m_v\}
        \]
        where $m_v$ is a nonnegative integer depending on $v$, and where each $F_i$ is a sensing function, with $x_i$ in the domain of $F_i$ and $y_i$ in the codomain of $F_i$. We allow $L_v$ to be the empty set.
        \item[WD2.] There is a labeling of the vertices, given by a function $f : V \to \{1, 2, \cdots, n\}$ where $n$ is the number of elements in $V$, such that for each $a \in A$, we have $f(s(a)) < f(t(a))$.
    \end{enumerate}
\end{defn}

Wiring diagrams extend the knowledge-representation framework of ologs laid out in \cite{SpivakKent}, and provide  a way to represent abstract concepts that have a temporal component.  Wiring diagrams also allow us to formally quantify the analogy between two abstract concepts \cite[Section 6]{LoMR1}.

The  underlying graph of a wiring diagram is called a wiring diagram graph.  Equivalently, we can define it as follows:

\begin{defn}[wiring diagram graph]
A \emph{wiring diagram graph} (or \emph{WD graph}) is a directed graph $G=(V,A,s,t)$ satisfying WD2.
\end{defn}

\begin{rem}
Note that condition WD2 is equivalent to having a linear extension ordering for a directed graph, and so a directed graph is a wiring diagram graph if and only if it is a directed acyclic graph (DAG) \cite[Remark 5.8]{LoMR1}.  In this article, however, we will continue to use the term `wiring diagram graph' to emphasize the perspective  that these are the underlying graphs of wiring diagrams. 
\end{rem}

In \cite{LoMR1}, the first author defined a \emph{skeleton WD graph} $G$ as a WD graph satisfying the following condition:
\begin{itemize}
    \item[WD3.] For any two vertices $v, v'$ in $G$, there is at most one path from $v$ to $v'$.
\end{itemize}

\begin{defn}[quasi-skeleton WD graph]\label{def:qsWDg}
We say a wiring diagram graph $G$ is \emph{quasi-skeleton} if it satisfies:
\begin{itemize}
    \item[(a)] For any two distinct vertices $v, v'$ in $G$, there is at most one arrow from $v$ to $v'$.
    \item[(b)] For any two distinct vertices $v, v'$ in $G$, if there is a path of length at least 2 from $v$ to $v'$, then there is no arrow  from $v$ to $v'$.
\end{itemize}
We say a wiring diagram is quasi-skeleton if its underlying graph is quasi-skeleton.
\end{defn}
It is clear that every skeleton WD graph is a quasi-skeleton WD graph.  

In a wiring diagram, an arrow from $A$ to $B$ indicates that the event at vertex $A$ should occur before the event at vertex $B$.  As a result, we only study quasi-skeleton WD graphs in this article. 


\begin{eg}
Consider the WD graph with $A, B, C, D$ as vertices
\[
\xymatrix @R=-0.1pc {
& B \ar[dr] & \\
A \ar[ur] \ar[dr] & & D \\
& C \ar[ur] & 
}
\]
This WD graph is not skeleton since there are two distinct paths from $A$ to $D$.  On the other hand, it is quasi-skeleton since, despite having two paths of length at least 2 from $A$ to $D$, there are no arrows from $A$ to $D$.
\end{eg}

\section{Wiring diagram graphs vs Hasse diagrams}\label{sec:WDvsHasse}


\paragraph[Transitive closure] \label{para:trans-clos} Following the ideas in \cite[Section 2]{aho1972transitive}, for every directed graph $G$ satisfying the condition
\begin{itemize}
    \item for any two vertices $v, v'$ (not necessarily distinct) in the graph, there is at most one arrow from $v$ to $v'$
\end{itemize}
(e.g.\ a quasi-skeleton WD graph),   if we write $V$ to denote the set of vertices of $G$, then we can think of $G$ as a subset of $V \times V$.  In this perspective, an element $(v,v')$ of $G$ represents an arrow from $v$ to $v'$.  Then we say $G$ is \emph{transitive} if, for every pair of vertices $v, v'$ (not necessarily distinct), there is a path from $v$ to $v'$  if and only if $(v,v') \in G$.
We define the \emph{transitive closure} $G^T$ of $G$ to be the smallest subset of $V \times V$ that contains $G$ and is transitive. 

\paragraph[Transitive reduction] \label{para:trans-red} 

\begin{defn}[transitive reduction, version 1]\cite[Section 1]{aho1972transitive}\label{def:TR-v1}
Given a directed graph $G$, the \emph{transitive reduction} $G^t$ of $G$  is a graph with the same vertex set as $G$ satisfying the following two conditions:
\begin{itemize}
    \item[(i)] There is a path from a vertex $v$ to a vertex $v'$ in $G^t$ if and only if there is a path from $v$ to $v'$ in $G$.
    \item[(ii)] There is no graph with fewer arrows than $G^t$ satisfying condition (i).
\end{itemize}
\end{defn}

For a finite acyclic directed graph $G$, the transitive reduction $G^t$ can be realized as a subgraph of $G$ by removing redundant arrows in any order and is unique \cite[Section 2]{aho1972transitive}, and coincides with the notion of a `minimum equivalent' of $G$ as defined in \cite{moyles1969algorithm} (see also \cite[Section 1]{aho1972transitive}).

Another definition of transitive reduction can be found in \cite[Section 2.2]{CHD}.  To state this definition, recall that in a poset $(P,\leq)$, we say an element $v$ covers $u$ if $u < v$ and there is no element $x$ such that $u<x<v$.  

In this article, we say a directed graph $G$ \emph{corresponds to a poset} if there is at most one arrow from any vertex $v$ to any (not necessarily distinct) vertex $v'$ in $G$ and, if we let $V$ denote the set of vertices of $G$, then the subset of $V \times V$
\[
\{ (a,b) \in V \times V : G \text{ has an arrow from $a$ to $b$} \}
\]
is a poset.   

\begin{defn}[transitive reduction, version 2]\cite[Section 2.2]{CHD}\label{def:TR-v2}
Suppose $G$ is a directed graph corresponding to a poset.  The \emph{transitive reduction} $G'$ of $G$ is a directed acyclic graph $G'$ with the same vertex set as $G$ such that there is an edge from $u$ to $v$ in $G'$ if and only if $v$ covers $u$ in $G$.
\end{defn}

The equivalence between Definition \ref{def:TR-v1} and Definition \ref{def:TR-v2} appear to be well known.  Nonetheless, we include here a proof for the case of graphs corresponding to posets  for ease of reference.  Recall that a \emph{loop} in a directed graph is an arrow from a vertex to itself.

\begin{lem}\label{lem:ON-MR06p2-16-1}
Let $G$ be a directed graph corresponding to a poset, and let $G^\circ$ denote the graph obtained from $G$ by removing the unique loop at every vertex.  Then the transitive reduction $(G^\circ)^t$ of $G^\circ$ in the sense of Definition \ref{def:TR-v1} coincides with the transitive reduction $G'$ of $G$ in the sense of Definition \ref{def:TR-v2}.
\end{lem}

\begin{proof}
Let $G$ be a graph corresponding to a poset.  Note that neither $(G^\circ)^t$ nor $G'$ has any loop.  So it suffices to show that for any two distinct vertices $u, v$ in $G$, there is an arrow from $u$ to $v$ in $(G^\circ)^t$ if and only if there is an arrow from $u$ to $v$ in $G'$.

Suppose there is an arrow  $\alpha : u \to v$ in $(G^\circ)^t$.  By \cite[Theorem 1]{aho1972transitive}, $(G^\circ)^t$ can be obtained from $G^\circ$, hence $G$ itself, by deleting arrows and loops. If $v$ does not cover $u$ in $G$, then there exists a vertex $w$, distinct from $u$ and $v$, such that there are arrows $u \to w \to v$ in $G$.  This means that there must be a path $p$ from $u$ to $w$, and a path $p'$ from $w$ to $v$ in $(G^\circ)^t$; the concatenation $p.p'$ then gives a path from $u$ to $v$ in $(G^\circ)^t$, meaning $\alpha$ could be removed from $(G^\circ)^t$ while preserving condition (i) in Definition \ref{def:TR-v1}, thus contradicting the minimality of $(G^\circ)^t$ required by condition (ii).  Hence $v$ covers $u$ in $G$, and there is  an arrow $u \to v$ in $G'$.

Conversely, suppose $\alpha : u \to v$ is an arrow in $G'$, i.e.\ $v$ covers $u$ in $G$.  This means there is a path of length $1$ from $u$ to $v$ in $G$.  As a result, there must be a path from $u$ to $v$ in $(G^\circ)^t$; let $q$ be such a path of minimal length, say 
\[
u=x_0 \to x_1 \to \cdots \to x_{n-1} \to x_n = v
\]
where $n \geq 1$.  If $n \geq 2$, then since $q$ is also a path in $G$, it follows that $v$ does not cover $u$, a contradiction.  Hence $n=1$ and there is an arrow from $u$ to $v$ in $(G^\circ)^t$.
\end{proof}


\begin{lem}\label{lem:quaskegr-tranred}
Every quasi-skeleton WD graph $G$ is the transitive reduction of $G^T$ in the sense of Definition \ref{def:TR-v1}.
\end{lem}

\begin{proof}
Let $G$ be any quasi-skeleton WD graph.  By the construction of $G^T$, we have that $G$ is a subgraph of $G^T$ with the same reachability as $G^T$.  Therefore, to prove that $G$ is the transitive reduction of $G^T$ in the sense of Definition \ref{def:TR-v1}, it suffices to show that removing any arrow from $G$ would cause its reachability to become different from that of $G^T$.

Let $\alpha$ be any arrow in $G$, say from vertex $u$ to vertex $v$, and let $G^-$ denote the graph obtained from $G$ by deleting $\alpha$.  For the sake of contradiction, suppose $G^-$ has the same reachability as $G$.  Then there must be a path $p$ from $u$ to $v$ in $G^-$, and this path is different from the path from $u$ to $v$ given by $\alpha$ itself in the graph $G$.  By condition (a) in Definition \ref{def:qsWDg}, $p$ must be a path of length at least 2; but then the existence of $\alpha$ and $p$ in $G$ contradicts condition (b) in Definition \ref{def:qsWDg}.  Hence $G$ is the transitive reduction of $G^T$.
%
%
\end{proof}

\begin{thm}\label{thm:ON-MR06p2-p15}
Let $V$ be a finite set.  Then a graph $G$ is a quasi-skeleton WD graph with vertex set $V$ if and only if it is the transitive reduction of a poset on $V$.
\end{thm}

\begin{proof}
Suppose $G$ is a quasi-skeleton WD graph with vertex set $V$.  Since $G$ satisfies WD2, the transitive closure $G^T$ is acyclic and so the subset of $V \times V$
\[
G^\ast := G^T \cup \{ (v,v) : v \in V\}
\]
is a poset on $V$.  By Lemma \ref{lem:quaskegr-tranred}, $G$ is the transitive reduction of $G^T$ in the sense of Definition \ref{def:TR-v1}.  Since $G^T$ is just the graph obtained from $G^\ast$ by removing the loop at each vertex, Lemmas \ref{lem:ON-MR06p2-16-1} and \ref{lem:quaskegr-tranred} together imply that $G$ is the transitive reduction of the poset $G^\ast$ in the sense of Definition \ref{def:TR-v2}.


Conversely, suppose $G$ is the transitive reduction of a poset $G^\dagger$ on $V$. Then $G$ is a subgraph of $G^\dagger$; since every poset embeds into a total order by the order extension principle, $G$ itself satisfies WD2 and so is a WD graph.  Since $G$ is a subgraph of a poset, it satisfies condition (a) in Definition \ref{def:qsWDg}.  Now suppose there is a path $p$ of length at least 2 from vertex $u$ to vertex $v$ in $G$; then $p$ is also a path of length at least 2 in $G^\dagger$.  If there is also an arrow $\alpha : u \to v$ in $G$, then by Definition \ref{def:TR-v2}, $v$ must cover $u$ in $G^\dagger$, contradicting the existence of the path $p$ in $G^\dagger$. Hence condition (b) in Definition \ref{def:qsWDg} also holds, and $G$ is a quasi-skeleton WD graph.
\end{proof}

\begin{rem}\label{rmk:1-1corr-qsWDgHDp}
Transitive reductions of posets, drawn in a specific way, are called \emph{Hasse diagrams} (see \cite[D29, 3.2.3]{handGT} for a precise definition).  Therefore, Theorem \ref{thm:ON-MR06p2-p15} says that there is a 1-1 correspondence between quasi-skeleton WD graphs and Hasse diagrams of posets.
\end{rem}


\section{Category of quasi-skeleton WD graphs}\label{sec:catQSWDG}

\paragraph \label{para:RandRprime} Given any quasi-skeleton WD graph $G$ with a finite set of vertices $V$, we can identify $G$ with the subset of $V \times V$
\begin{multline*}
R'(G) := \{ (u, v ) \in V \times V : \\
\text{ there is an arrow $u \to v$ in $G$} \}.
\end{multline*}
In \cite[Section 6.1]{LoMR1}, we defined $R(G)$ to be the subset of $V \times V$ obtained by forcing reflexivity and transitivity on $R'(G)$. 

Note that for two quasi-skeleton WD graphs $G_1, G_2$ with the same set of vertices, if $R'(G_1) \subseteq R'(G_2)$ then  $R(G_1) \subseteq R(G_2)$.  The converse, however, is not true: consider the graph $G_1$ 
\begin{equation*}
\xymatrix @R=-0.1pc {
& B   \\
A  \ar[dr]  \\
& C 
}
\end{equation*}
and the graph $G_2$
\begin{equation*}
\xymatrix @R=-0.1pc {
& B   \ar[dd]\\
A  \ar[ur]  \\
& C 
}
\end{equation*}
Both are quasi-skeleton WD graphs and $R(G_1) \subseteq R(G_2)$, but $R'(G_1) \nsubseteq R'(G_2)$.

\begin{defn}[categories $\mathcal{R}(V)$ and $\mathcal{R}$]
Given a finite set $V$, we define $\mathcal{R}(V)$ to be the category where the objects are quasi-skeleton WD graphs with vertex set $V$, and where we have a morphism $G_1 \to G_2$ between two quasi-skeleton WD graphs if there is an inclusion of sets $R(G_2) \subseteq R(G_1)$.

We also  define $\mathcal{R}$ to be the category where the objects are quasi-skeleton WD graphs (on any finite set), and where we declare a morphism $G_1 \to G_2$ between two quasi-skeleton WD graphs if there is an inclusion of sets $R(G_2) \subseteq R(G_1)$.
\end{defn}

\subparagraph Given any quasi-skeleton WD graph $G$, we can informally think of the set $R(G)$ as the set of constraints represented by $G$.  Therefore, every time we have a morphism $G_1 \to G_2$ in the category $\mathcal{R}$, it means that the graph $G_1$ encompasses the same or more constraints than the graph $G_2$, allowing us to think of $G_2$ as a ``generalization''' of $G_1$, or that $G_2$ represents a concept that is more abstract than that represented by $G_1$.




\begin{eg}
Let $G_1$ denote the quasi-skeleton WD graph
\begin{equation*}
\xymatrix @R=-0.1pc {
& B\ar[dr] & \\
A \ar[ur] \ar[dr] & & D \\
& C \ar[ur]
}
\end{equation*}
with vertex set $\{A, B, C, D\}$, and let $G_2$ denote the quasi-skeleton WD graph
\begin{equation*}
\xymatrix @R=-0.1pc {
& B   \\
A  \ar[dr] \ar[ur]  \\
& C 
}
\end{equation*}
with vertex set $\{A, B, C\}$.  We have a morphism $G_1 \to G_2$ in $\mathcal{R}$ since $R'(G_2) \subseteq R'(G_1)$, even though $G_1$ and $G_2$ have different sets of vertices.
\end{eg}

\paragraph \label{para:catRVHasseDseq} By Theorem \ref{thm:ON-MR06p2-p15} and Remark \ref{rmk:1-1corr-qsWDgHDp}, given a finite set $V$, the set of objects in the category $\mathcal{R}(V)$ is in 1-1 correspondence with the set of Hasse diagrams on the set $V$.  Note that the number of Hasse diagrams on a finite set forms the sequence A001035 in the On-Line Encyclopedia of Integer Sequences \cite{OEISposet}.  

By abuse of notation, let us also write $\mathcal{R}(V)$ to denote the set of objects in the category $\mathcal{R}(V)$.  Then we obtain a poset $(\mathcal{R}(V),\leq)$ by defining $G_1 \leq G_2$ iff $R(G_1) \subseteq R(G_2)$.  This gives a different way of thinking about the category $\mathcal{R}(V)$.  


The following lemma will be useful later on.  We omit the proof since it is clear.

\begin{lem}\label{lem:mtxdiffvsgraphmor}
Let $J= \{e_1, \cdots, e_m\}$ be a finite set, and let $G_1, G_2$ be two quasi-skeleton WD graphs with $J$ as the set of vertices.  For $k=1,2$, let $M_k$ denote the path matrix of $G_k$, i.e. $M_k(i,j)=1$ (resp.\ $0$) if and only if there is a path (resp.\ no paths) from $e_i$ to $e_j$ in $G_k$.  Then the following are equivalent:
\begin{itemize}
    \item[(i)] All the entries of $M_2-M_1$ are nonnegative.
    \item[(ii)] $R(G_1) \subseteq R(G_2)$.
    \item[(iii)] There is a morphism from $G_2$ to $G_1$ in the category $\mathcal{R}$.
\end{itemize}

\end{lem}

\section{Consistent wiring diagrams}\label{sec:consistWD}





We begin by fixing some terminology that will facilitate the discussion that follows.  

We say a finite directed graph $G$ is a \emph{path graph} if it is of the form
\[
\xymatrix{
\bullet \ar[r] & \bullet \ar[r] & \cdots \ar[r] & \bullet
}
\]

\begin{defn}\label{defn:simple}
Let $X$ be a finite set.
\begin{itemize}
    \item Give a set $X$, a \emph{sequence of subsets of $X$} is a sequence $S=(S_i)_{i \geq 1}$
    where each term $S_i$ is a subset of $X$.  
    \item A \emph{sequence in $X$} is a sequence $(s_i)$ where each term $s_i$ is an element of $X$.
    \item A \emph{simple} sequence of subsets of $X$ is a sequence $(S_i)$ of subsets of $X$ such that the sets $S_i$ are  pairwise disjoint.
    \item A \emph{simple} sequence in $X$ is a sequence $(s_i)$ in $X$ such that all the terms are distinct.
\end{itemize}
\end{defn}

\paragraph \label{para:seqsingletonconv} Given any sequence  in $X$, say
\[
 s_1, s_2, \cdots
\]
we can construct a sequence $(S_i)$ of subsets of $X$  by replacing each term with the singleton set containing that term:
\[
  \{s_1\}, \{s_2\}, \cdots
\]
i.e.\ by taking $S_i = \{s_i\}$ for all $i$.  Conversely, given a sequence $(S_i)$ of subsets of $X$ where each term $S_i$ is a singleton set, say 
\[
\{s_1\}, \{s_2\}, \cdots
\]
we can associate a sequence in $X$  by replacing each term with its unique element:
\[
s_1, s_2, \cdots
\]
Using the above constructions, we will sometimes confuse sequences of singleton subsets of $X$ with sequences in $X$.  Through the above construction, definitions for sequences of subsets of $X$ will also apply to sequences in $X$. 

Sometimes, given a sequence of subsets of $X$ or a sequence in $X$, we want to extract information that focuses on a specific collection of elements.  

\begin{defn}
Let $X$ be a finite set, and $I$ a nonempty subset of $X$. Given a sequence $S = (S_i)$ of subsets of $X$, we define $m_I(S)$ to be the sequence of subsets of $X$ obtained by replacing each term $S_i$ with the set $I \cap S_i$, and then removing any term that is the empty set.

Given a sequence $s=(s_i)$ in $X$, we similarly define $m_I(s)$ by  regarding $s$ as a sequence of singleton subsets of $X$ as in \ref{para:seqsingletonconv}.  Equivalently, we define $m_I(s)$ to be the sequence obtained from $s$ by removing any term that is not in $I$.
\end{defn}


\begin{eg}
Suppose $X=\{A, B, C, D, E, F, G, H\}$ and $I=\{A, B, C, D\}$.  If $S$ denotes the sequence of subsets of $A$
\[
\{C\}, \{A, G, H\}, \{G\}, \{B, D, H\}
\]
then $m_I(S)$ is the sequence
\[
\{C\}, \{A\}, \{B, D\}.
\]
If $s$ denotes the sequence in $X$
\[
G, B, C, E, A
\]
then $m_I(s)$ denotes the sequence
\[
B, C, A.
\]
\end{eg}

\begin{defn}[simple wiring diagram]
We say a wiring diagram $W$ is \emph{simple} if it satisfies the following two properties:
\begin{itemize}
    \item[(i)] Every vertex in $W$ has exactly one label.
    \item[(ii)] All the vertex labels in $W$ are distinct.
\end{itemize}
\end{defn}

\begin{defn}[consistent wiring diagram]\label{def:consWD}
Let $X$ be a finite set.  Suppose $S=(S_i)$ is a  sequence of subsets of $X$, and $W$ is a simple, quasi-skeleton wiring diagram where the set $I$  of vertex labels is a subset of $X$.  We say $S$ is \emph{consistent} with $W$ (or that $W$ is consistent with $S$) if the following condition holds:
\begin{itemize}
    \item For any $x, y \in (\cup_i S_i) \cap I$, say $x \in S_{i_1}$ and $y \in S_{i_2}$ and $x,y$ are the labels of vertices $u, v$ in $W$, respectively, we have $i_1 < i_2$ whenever there is a path from $u$ to $v$ in $W$.
\end{itemize}
\end{defn}

That is, we say $S$ is consistent with $W$ if, whenever there is a path from $u$ to $v$ in $W$,  every instance of the vertex label at $u$ appears before every instance of the vertex label at $v$ within  the sequence $S$.  Informally, $S$ is consistent with $W$ if all the relations on vertex labels of $W$ imposed by the arrows of $W$ are respected by $S$.

By regarding a sequence in $X$ as a sequence of subsets of $X$ via the construction in \ref{para:seqsingletonconv}, we can also define what it means for a sequence in $X$ to be consistent with a simple, quasi-skeleton wiring diagram.

\begin{eg}
Suppose $X =\{A, B, C, D, E\}$ and $I=\{A, B, C\}$.
Let $W$ denote the wiring diagram
\begin{equation*}
\xymatrix @R=-1pc {
& {\begin{matrix} \bullet \\ B \end{matrix}}   \\
{\begin{matrix} \bullet \\ A \end{matrix}}  \ar@<1ex>[dr] \ar@<1ex>[ur]  \\
& {\begin{matrix} \bullet \\ C \end{matrix}} 
}
\end{equation*}
which is both simple and quasi-skeleton and where all the vertex labels belong to $I$.  The simple sequence of subsets of $X$
\[
\{D\}, \{A, E\}, \{B, C\}
\]
is consistent with $W$, whereas the sequence
\[
\{D\}, \{A, B\}, \{C, E\}
\]
is not consistent with $W$.
\end{eg}

In Lemma \ref{lem:WDconsisreform-1}  below, we will  give an alternative formulation of Definition \ref{def:consWD} in terms of relations in posets.  To do this, we need to define some operations that convert  wiring diagrams, sequences, and graphs among one another.

\begin{defn}[graph $\wtg(W)$ of a wiring diagram $W$]\label{def:wtg}
Let $W$ be a simple, quasi-skeleton wiring diagram.  Let $V$ denote the set of vertices of $W$ and, for each $v \in V$, let $s_v$ denote the unique label at the vertex $v$.  We define $\wtg(W)$ to be the directed graph where the set of vertices is $\{s_v\}_{v \in V}$, and where there is an arrow $s_u \to s_v$ if and only if there is an arrow $u \to v$ in the underlying graph of $W$.
\end{defn}

Since $W$ is a simple wiring diagram in the definition above, one can think of $\wtg(W)$ as the graph obtained from $W$ by  replacing every vertex $v$ with its  label $s_v$.

\begin{defn}[graph $\stg(S)$ of a sequence $S$]\label{defn:gra-op}
Let $X$ be a finite set, and let $S=(S_i)_{1 \leq i \leq n}$ be  a finite, simple sequence of nonempty subsets of $X$.  We define $\stg (S)$ to be the graph $G$ constructed as follows:
\begin{itemize}
    \item[(i)] The set of vertices is given by the disjoint union $\coprod_{i=1}^n S_i$.
    \item[(ii)] For each $1 \leq i \leq n-1$, there is a single arrow from each element of $S_i$ to each element of $S_{i+1}$.
\end{itemize}
Note that $\stg(S)$ is a quasi-skeleton WD graph by construction.
\end{defn}

\begin{eg}
Suppose $X=\{A, B, C, D, E, F\}$ and $S$ is the finite simple sequence of subsets of $X$
\[
  \{A\}, \{B, C\}, \{D, E\}, \{F\}
\]
Then $\stg(S)$ is the graph
\[
\xymatrix @R=-0.1pc {
& B \ar[r] \ar[ddr] & D \ar[dr] & \\
A \ar[ur] \ar[dr] & & & F \\
& C \ar[uur] \ar[r] & E \ar[ur] &
}
\]
\end{eg}

We can easily define an inverse operation of $\stg (-)$ on graphs that arise under the operation $\stg(-)$:

\begin{defn}\label{defn:gts-op}
Suppose $G$ is a directed graph where
\begin{itemize}
    \item[(i)] There is a partition $V = V_1 \coprod \cdots \coprod V_n$ of the set $V$ of vertices where  the $V_i$ are all nonempty.
    \item[(ii)] For each $1 \leq i \leq n-1$, there is a single arrow from each element of $V_i$ to each element of $V_{i+1}$, and these are all the arrows in $G$.
\end{itemize}
We define $\gts (G)$ to be the sequence
\[
 V_1, V_2, \cdots, V_n
\]
of subsets of $V$.
\end{defn}

In particular, any path graph satisfies both conditions (i) and (ii) in Definition \ref{defn:gts-op}.  If $P$ is the path graph
\[
  \xymatrix{
  A_1 \ar[r] & A_2 \ar[r] & \cdots \ar[r] & A_n
  }
\]
with vertices $A_1, \cdots, A_n$ and arrows as shown, then  $\gts (P)$ is the sequence 
\[
 \{A_1\}, \{A_2\}, \cdots, \{A_n\}.
\]

The definition of a consistent wiring diagram in Definition \ref{def:consWD} can now be rephrased for finite, simple sequences of subsets of $X$.  

\begin{lem}\label{lem:WDconsisreform-1}
Let $X$ be a finite set.  Suppose $S=(S_i)$ is a finite, simple sequence of subsets of $X$, and $W$ is a simple, quasi-skeleton wiring diagram such that the set $I$  of vertex labels is a subset of $X$.   Set
\[
U = \{ l \in I : l \text{ appears in } S\} = (\cup_i S_i) \cap I.
\]
Then  $S$ is consistent with $W$ if and only if
\begin{equation}\label{eq:MR06p2-18-1}
 R(\wtg(W)) \cap (U \times U) \subseteq R(\stg(m_I(S))).
\end{equation}
\end{lem}

\begin{proof}
Note that for every $u \in U$, the ordered pair $(u,u)$ lies in both the left-hand side and right-hand side of \eqref{eq:MR06p2-18-1}.  

Now take any distinct elements $u, v \in U$.  By the construction of the operation $R(-)$ (see \ref{para:RandRprime}), there is a path from $u$ to $v$ in $W$ if and only if $(u,v)$ is an element of the left-hand side of \eqref{eq:MR06p2-18-1}.  On the other hand, observe that $u, v$ both appear in the sequence of subsets $m_I(S)$.  Suppose $u \in S_{i_1}$ and $v \in S_{i_2}$.  Then $i_1 < i_2$ if and only if there is a path from $u$ to $v$ in $\stg(m_I(S))$, if and only if $(u,v)$ is an element of the right-hand side of \eqref{eq:MR06p2-18-1}.  The lemma then follows.
\end{proof}


\begin{defn}[restriction of quasi-skeleton WD graph to a subset of vertices]\label{def:Gressubver}
Suppose $G$ is a quasi-skeleton WD graph with set of vertices $V$.  For any subset $U \subseteq V$, we define the graph $G_U$ to be the transitive reduction (see Definition \ref{def:TR-v2}) of the poset
\[
    R(G) \cap (U \times U).
\]
\end{defn}

\begin{rem}\label{rem:Gressubver-1}
In Definition \ref{def:Gressubver}, note that:
\begin{itemize}
    \item[(a)]  $R(G) \cap (U \times U) = R(G_U)$.  In particular, this implies $R(G_U) \subseteq R(G)$, i.e.\ there is a morphism $G \to G_U$ in the category $\mathcal{R}$.
    \item[(b)] For any $u, v \in U$, there is a path from $u$ to $v$ in  $G_U$ if and only if there is a path from $u$ to $v$ in $G$.
    \item[(c)] $G_U$ is again a quasi-skeleton WD graph.  This is because $G_U$ is the transitive reduction of a poset by construction, and so is a quasi-skeleton WD graph by Theorem \ref{thm:ON-MR06p2-p15}.
\end{itemize} 
\end{rem}

We now rephrase Lemma \ref{lem:WDconsisreform-1} in terms of the category $\mathcal{R}$.


\begin{cor}\label{cor:WDconsisreform-1-1}
Assume the setup of Lemma \ref{lem:WDconsisreform-1}.
Then  $S$ is consistent with $W$ if and only if there is a morphism
\[
\stg (m_I(S)) \to (\wtg (W))_U
\]
in the category $\mathcal{R}$.
\end{cor}

\begin{proof}
By Lemma \ref{lem:WDconsisreform-1}, $S$ is consistent with $W$ if and only if 
\[
R(\wtg (W)) \cap (U \times U) \subseteq R(\stg (m_I(S))),
\] 
i.e.\
\[
 R ( (\wtg (W))_U) \subseteq R(\stg (m_I(S)))
\]
by Remark \ref{rem:Gressubver-1}(a).  From the definition of the category $\mathcal{R}$, this is equivalent to having a morphism 
\[
 \stg (m_I(S)) \to ( \wtg (W))_U
\]
in $\mathcal{R}$.
\end{proof}

We  now prove a universal property for a certain type of path graphs in the category $\mathcal{R}$.

\begin{prop}\label{prop:ON06p2-23-1}
Suppose $P$ is a path graph and $G$ is a quasi-skeleton WD graph such that the vertex set $I$ of $G$ is a subset of the vertex set of $P$.  Then any morphism $\alpha : P \to G$ in the category $\mathcal{R}$ factors uniquely through the natural map $\beta : P \to P_I$:
\[
\xymatrix{
  P \ar[rr]^{\alpha} \ar[dr]_\beta & & G \\
  & P_I \ar[ur] & 
}
\]
\end{prop}

Note that $P_I$, as defined in Definition \ref{def:Gressubver},   is merely the path graph obtained from $P$ by deleting all the vertices not in $I$; it  can also be written as $\stg (m_I (\gts (P)))$.  The morphism $\beta$ corresponds to the inclusion 
\[
R(P_I) \subseteq R(P).
\]

\begin{proof}
Take  any morphism $\alpha : P \to G$ in $\mathcal{R}$.  This means that there is an inclusion of sets $R(G) \subseteq R(P)$.  Note that $R(G) \subseteq I \times I$, and so we have
\[
R(G) \subseteq R(P) \cap (I \times I) \subseteq R(P).
\]
Also note 
\[
R(P_I) = R(P) \cap (I \times I)
\]
by Remark \ref{rem:Gressubver-1}(a).  Hence the inclusions above can be rewritten as 
\[
R(G) \subseteq R(P_I)  \subseteq R(P)
\]
which gives the desired factorization of $\alpha$.
\end{proof}

\begin{rem}\label{rem:graphflattening-char}
Given a path graph $P$ and a quasi-skeleton WD graph $G$, if the vertex set of $G$ is not a subset of that of $P$, then there cannot be any morphism $P \to G$ in the category $\mathcal{R}$.  In Proposition \ref{prop:ON06p2-23-1}, note that the vertex set of $P_I=\stg (m_I (\gts (P)))$ is exactly $I$.  
Therefore, Proposition \ref{prop:ON06p2-23-1}  tells us that any morphism from a path graph $P$ to a quasi-skeleton WD graph $G$ must factor through \emph{a path graph whose vertex set is the same as that of $G$}.   As a consequence, path graphs that have the same vertex set as a given wiring diagram 
 play an important role in understanding the connection between data  (represented by $P$) and wiring diagrams (with underlying graph  $G$).
\end{rem}


Remark \ref{rem:graphflattening-char} motivates us to make Definitions \ref{def:flattening-graph} and \ref{def:flattening} below.

\begin{defn}[flattening of a quasi-skeleton WD graph]\label{def:flattening-graph}
Suppose $G$ is a quasi-skeleton WD graph with set of vertices $V$.  We say a graph $P$ is a \emph{flattening} of $G$ if the following conditions are satisfied:
\begin{itemize}
    \item[(i)] $P$ is a path graph.
    \item[(ii)] The set of vertices of $P$ is also $V$.
    \item[(iii)] There is a morphism $P \to G$ in the category $\mathcal{R}$.
\end{itemize}
\end{defn}

\begin{defn}[flattening of a wiring diagram]\label{def:flattening}
Suppose $W$ is a quasi-skeleton wiring diagram with set of vertices  $V$.  We say a wiring diagram $W'$ is a \emph{flattening} of $W$ if the following are satisfied:
\begin{itemize}
    \item[(a)] The underlying graph of $W'$ is a flattening of the underlying graph of $W$.
    \item[(b)] For each $v \in V$, the set of vertex labels at $v$ in $W$ coincides with the set of vertex labels at $v$ in $W'$.
\end{itemize}
\end{defn}

Note that if $W$ is a simple quasi-skeleton wiring diagram, then any flattening of $W$ is also a simple quasi-skeleton wiring diagram.


Intuitively, a flattening $W'$ of a wiring diagram $W$ is a possible outcome when we try to `flatten' $W$ into a 1-dimensional diagram.  Of course, there can be numerous possible flattenings of the same wiring diagram $W$.


\begin{eg}\label{eg:flatteningsofgraph}
Suppose $G$ is the quasi-skeleton WD graph 
\begin{equation*}
\xymatrix @R=-0.1pc {
& B\ar[dr] & \\
A \ar[ur] \ar[dr] & & D \\
& C \ar[ur]
}
\end{equation*}  Then $G$ has two possible flattenings: 
\begin{equation*}
\xymatrix{
A \ar[r] & B \ar[r] & C\ar[r] & D
}
\end{equation*}
and 
\begin{equation*}
\xymatrix{
A \ar[r] & C \ar[r] & B\ar[r] & D
}
\end{equation*}
\end{eg}


Let us now introduce the analogue of  Definition \ref{def:Gressubver} in the context of wiring diagrams:

\begin{defn}[restriction of wiring diagram to a subset of vertices]\label{def:Wressubver}
Suppose $W$ is a quasi-skeleton wiring diagram with underlying graph $G$ and set of vertices $V$.  For any subset $U \subseteq V$, we define the wiring diagram $W_U$ to be such that:
\begin{itemize}
    \item[(i)] The underlying graph of $W_U$ is $G_U$ as constructed in Definition \ref{def:Gressubver}.
    \item[(ii)] For every $u \in U$, the set of vertex labels at $u$ in $W_U$ is the same as the set of vertex labels at $u$ in $W$.
\end{itemize}
\end{defn}


\paragraph \label{para:setup-flattenunivprop} Consider the following setup.
\begin{itemize}
    \item $X$ is a finite set.
    \item $s=(s_i)_{1 \leq i \leq q}$ a finite simple sequence in $X$.
    \item $W$ is a simple quasi-skeleton wiring diagram.
    \item $G$ is the underlying graph of $W$.
    \item $V$ is the vertex set of $G$.
    \item For each $v \in V$, $l_v \in X$ is the vertex label at $v$ in $W$.
    \item $\widetilde{V} := \{ l_v : v \in V\}$ is the set of vertex labels in $W$.
    \item $U :=  \{ v \in V : l_v \text{ appears in } s\}$ is the set of vertices in $W$ where the corresponding labels  appear in the sequence $s$.
     \item $\widetilde{U} := \{ l_u : u \in U\}$ is the set of vertex labels  in $W$ that appear in $s$.
\end{itemize}

In the following theorem, we collect the various ways of  seeing whether a sequence (representing data) is consistent with a wiring diagram (representing a concept).

\begin{thm}\label{thm:ON-MR06p2-20-1}
Assume the setup and notation of \ref{para:setup-flattenunivprop}.  The following are equivalent:
\begin{itemize}
    \item[(i)] $s$ is consistent with $W$.
    \item[(ii)] $R(\wtg (W)) \cap (\widetilde{U} \times \widetilde{U}) \subseteq R(\stg (m_{\widetilde{V}}(s))$.
    \item[(iii)] There is a morphism $(\stg (s))_{\widetilde{U}} \to (\wtg (W))_{\widetilde{U}}$ in $\mathcal{R}$.
    \item[(iv)] There is a morphism $\stg (s) \to \wtg (W_U)$ in $\mathcal{R}$.
    \item[(v)] There is a morphism  $\stg (s) \to Z$ in $\mathcal{R}$ for some flattening $Z$ of $\wtg (W_U)$.
\end{itemize}
\end{thm}


\begin{proof}
The equivalence between (i) and (ii) was proved in Lemma \ref{lem:WDconsisreform-1}.

By Corollary \ref{cor:WDconsisreform-1-1}, (i) is equivalent to the existence of a morphism $\stg (m_{\widetilde{V}}(s)) \to (\wtg (W))_{\widetilde{U}}$ in $\mathcal{R}$.  Since $m_{\widetilde{U}}(s)=m_{\widetilde{V}}(s)$, we have $\stg(m_{\widetilde{V}}(s))=\stg (m_{\widetilde{U}}(s))$.  We also have $\stg (m_{\widetilde{U}}(s))= (\stg (s))_{\widetilde{U}}$.  The equivalence between (i) and (iii) then follows.

Let us assume (iii) holds.  Note that  $(\wtg (W))_{\widetilde{U}} = \wtg (W_U)$.   Pre-composing the morphism in (iii) with the natural map $\stg (s) \to (\stg (s))_{\widetilde{U}}$  gives the composition
\[
\stg (s) \to (\stg (s))_{\widetilde{U}} \to \wtg (W_U)
\]
in $\mathcal{R}$.  Conversely, suppose there is a morphism $\stg (s) \to \wtg (W_U)$ in $\mathcal{R}$.  Note that $\wtg(W_U)$ is a quasi-skeleton WD graph by Remark \ref{rem:Gressubver-1}(c).  Since the set of vertices of $\wtg (W_U)$ is $\widetilde{U}$, which is a subset of the set of vertices of $\stg (s)$, Proposition \ref{prop:ON06p2-23-1} applies and says that this morphism factors as 
\[
\stg (s) \to (\stg (s))_{\widetilde{U}} \to \wtg (W_U).
\]
Hence (iii) and (iv) are equivalent.

Assuming (iv), it is straightforward to see that  $(\stg (s))_{\widetilde{U}}$ is a flattening of $\wtg (W_U)$; the second half of the argument in the previous paragraph then shows that  (v) follows.  That (v) implies (iv) follows immediately from condition (iii) in Definition \ref{def:flattening-graph}.
\end{proof}

\paragraph \label{para:mainthmdiscussion} Assuming the setup of \ref{para:setup-flattenunivprop}, we have the natural morphisms  $\stg(s) \overset{\beta}{\to} (\stg (s))_{\widetilde{U}}$ and $\wtg (W) \overset{\gamma}{\to} (\wtg (W))_{\widetilde{U}}$ in $\mathcal{R}$ by Remark \ref{rem:Gressubver-1}(a).  The equivalence of (i) and (iii) in Theorem \ref{thm:ON-MR06p2-20-1} says that $s$ is consistent with $W$ if and only if there exists a morphism $\delta$ in $\mathcal{R}$ as shown:
\begin{equation}\label{eq:ON06p2-28-1}
\xymatrix{
  \stg (s) \ar[r]^\beta & (\stg (s))_{\widetilde{U}} \ar@{.>}[d]^\delta \\
  \wtg (W) \ar[r]^\gamma & (\wtg (W))_{\widetilde{U}}
}
\end{equation}
From the proof of the theorem, we know $(\stg (s))_{\widetilde{U}}$ is a flattening of $(\wtg (W))_{\widetilde{U}}$.


In applications, given a sequence $s$ that represents data collected from the sensors of a machine, we would want to identify all possible wiring diagrams $W$ representing concepts that are consistent with $s$.  If one has already determined the set $\widetilde{U}$ of vertex labels of $W$ that appear in $s$, then the morphism $\beta$ in \eqref{eq:ON06p2-28-1} is already determined, and so the problem of finding a $W$ that is consistent with $s$ becomes the problem of finding a wiring diagram $W$ with a  restriction $W_U$ such that $\wtg(W_U)=(\wtg(W))_{\wt{U}}$ admits $(\stg (s))_{\widetilde{U}}$ as a flattening.  This means, that in practice, finding all wiring diagrams that are consistent with a given sequence depends on  solving the following problem:

\begin{prob}\label{prob:flattening2WDs}
Given a path graph $P$, find all wiring diagrams $W$ such that for some restriction $W_U$ of $W$, the graph $\wtg (W_U)$ admits $P$ as a flattening.
\end{prob}

\subparagraph \label{para:mainthmdiscussion-b} In particular, when all the vertex labels of $W$ already appear in the sequence $s$, i.e.\ when $\widetilde{V}$ is a subset of $\{ s_i\}_{1 \leq i \leq q}$, the map $\gamma$ is an equality.  In this case,  diagram \eqref{eq:ON06p2-28-1} simplifies to
\begin{equation}\label{eq:ON06p2-28-1b}
\xymatrix{
  \stg (s) \ar[r]^\beta & (\stg (s))_{\widetilde{U}} \ar@{.>}[d]^\delta \\
  &  \wtg (W)
}
\end{equation}
while Problem \ref{prob:flattening2WDs} specializes to

\begin{prob}\label{prob:flattening2WDs-2}
Given a path graph $P$, find all quasi-skeleton WD graphs $G$   that  admits $P$ as a flattening.
\end{prob}

Of course, solving Problem \ref{prob:flattening2WDs} involves a finite search, so the issue lies not in the existence of a solution, but in designing an efficient algorithm for the search.


\begin{eg}
Let $P$ denote the path graph 
\begin{equation*}
\xymatrix{
A \ar[r] & B \ar[r] & C\ar[r] & D
}
\end{equation*}
and let $G$ denote the quasi-skeleton WD graph 
\begin{equation*}
\xymatrix @R=-0.1pc {
& B\ar[dr] & \\
A \ar[ur] \ar[dr] & & D \\
& C \ar[ur]
}.
\end{equation*} 
As we saw in Example \ref{eg:flatteningsofgraph}, $P$ is a flattening of $G$.   There are other graphs that admit $P$ as a flattening, however; one such example  is
\[
\xymatrix @R=-0.1pc {
  & & C \\
  A \ar[r] & B \ar[ur] \ar[dr] & \\
  & &  D
}
\]
\end{eg}

\section{Algorithms for understanding data}\label{sec:algos}

We now present an algorithm that converts a sequence  to a matrix.  The sequence that forms the input of the algorithm represents data, while the matrix that forms the output will be - in an ideal situation -  the path matrix of a WD graph, which in turn represents an abstract concept.  Therefore, this algorithm can be seen as a way to extract a concept from data.


\begin{algo}\label{algo:1-v4}
Let  $X$ be a finite set, $S=(S_i)_{1 \leq i \leq q}$ a finite sequence of subsets of $X$, and $J = \{ e_1, \cdots, e_m\}$ a subset  of $X$.  This algorithm returns  an $m \times m$ matrix $M_S$ with entries in $\{0, 1\}$.  Below, we will write $M_S(i,j)$ to denote the $(i,j)$-entry of $M_S$.

\begin{algorithmic}[1]
\State For each $1 \leq i \leq m$, define a set of positive integers 
\[
P_i := \{ k \in \mathbb{Z}^+ : e_i \in S_k \}.
\]
\State For each $1 \leq i \leq m$, set $M_S(i,i)=0$.
\State For every $1 \leq i \neq j \leq m$:
\If{$P_i, P_j$ are both nonempty and $\max{P_i} < \min{P_j}$}
  \State set $M_S(i,j)=1$
\Else
  \State set $M_S(i,j)=0$
\EndIf
\end{algorithmic}
\end{algo}

The output of Algorithm \ref{algo:1-v4} is a matrix $M_S$ such that, for any $1 \leq i, j \leq m$, the entry $M_S(i,j)$ equals 1 if and only if 
\begin{itemize}
    \item both $e_i, e_j$ appear in the sequence at least once, and
    \item every term in $S$ containing $e_i$ appears strictly before every term containing $e_j$. 
\end{itemize} 



\paragraph \label{para:algo:1-v4-disc1} To  see the connection between Algorithm \ref{algo:1-v4}  and the theory in Section \ref{sec:consistWD}, let us  assume the setup of \ref{para:setup-flattenunivprop}.  We can  regard $s$ as a simple sequence of subsets of $X$ as explained in \ref{para:seqsingletonconv}, and then apply  Algorithm \ref{algo:1-v4} by taking the input $S$ to be $s$ and the input $J$ to be $\widetilde{V}$.  Let us label the elements of $J=\widetilde{V}$ as $e_1, \cdots, e_m$ as in Algorithm \ref{algo:1-v4}.  The algorithm will then  produce an $m \times m$ matrix $M_s$.

If we now write $G$ to denote the   disjoint union of the graph $(\stg(s))_{\widetilde{U}}$ and the vertices $\wt{V} \setminus \wt{U}$, then $M_s$ is precisely the adjacency matrix of the transitive closure $G^T$ (or the \emph{path matrix} of $G$), i.e.\ $M_s(i,j)=1$ if and only if there is a path from $e_i$ to $e_j$ in $G$.  Note that $G$ is a quasi-skeleton WD graph even though it may not be connected.


\paragraph \label{para:algo:1-v4-disc2} Another way to understand the connection between  Algorithm \ref{algo:1-v4}  and  Section \ref{sec:consistWD} is as follows.  Suppose 
\[
\widetilde{U} = \{ e_{i_1}, e_{i_2}, \cdots, e_{i_n} \},
\]
i.e.\ suppose $e_{i_1}, e_{i_2}, \cdots, e_{i_n}$ with $i_1 < \cdots < i_n$ are the vertex labels in $W$ that appear in the sequence $s$.  Then we can extract an $n \times n$ matrix $\widehat{M_s}$  by omitting the $k$-th row and $k$-column in $M_s$ where $k$ runs over every element in
\[
\{1, 2, \cdots, m\} \setminus \{i_1, \cdots, i_n\}.
\]
Note that the elements in this set are precisely the indices $k$ for which $e_k$ does not appear in the sequence $s$.  The matrix $\widehat{M_s}$ is then the path matrix  of  $(\stg(s))_{\widetilde{U}}$.

Informally, \emph{any wiring diagram with $e_{i_1}, \cdots, e_{i_n}$ as vertex labels that is consistent with $s$  represents a concept that is  more abstract than that corresponding to the matrix $\widehat{M_s}$}.



In the special case where all the vertex labels of the wiring diagram $W$ appear in the sequence $s$,  we have $J=\wt{V}=\wt{U} = \{e_1, \cdots, e_m\}$.  Then  $M_s = \widehat{M_s}$ itself is  the path matrix   of $(\stg(s))_{\wt{U}}$.

Note that in the discussions in \ref{para:algo:1-v4-disc1} and  \ref{para:algo:1-v4-disc2}, the assumption of $s$ being a simple sequence  is not essential to our being able to  interpret the output $M_s$ of  Algorithm \ref{algo:1-v4} as the path matrix of a graph.  This is because of the following simple lemma:

\begin{lem}\label{lem:algv1hastrans}
Let $X$ be a finite set,  $s$ a finite sequence in $X$, and $\{e_1,\cdots, e_m\}$ a subset of $X$.  Let $M_s$ denote the corresponding output of Algorithm \ref{algo:1-v4}.  Then for any pairwise distinct indices $1\leq i, j, k \leq m$ such that $M_s(i,j)=1=M_s(j,k)$, we have $M_s(i,k)=1$.
\end{lem}

\begin{proof}
Using the notation in Algorithm \ref{algo:1-v4}, the assumption $M_s(i,j)=1=M_s(j,k)$ implies $P_i, P_j, P_k$ are all nonempty, and that 
\[
\max{P_i} < \min{P_j} \leq \max{P_j} < \min{P_k}.
\]
As a result, $\max{P_i}<\min{P_k}$ and hence $M_s(i,k)$ must be 1.
\end{proof}

We can slightly generalize Algorithm \ref{algo:1-v4} so that it takes in more than one sequence as input and produces a matrix as output:

\begin{algo}\label{algo:1-v4b}
Let  $X$ be a finite set, $S^{(1)}, \cdots, S^{(p)}$ finite sequences of  subsets of $X$, and $J = \{ e_1, \cdots, e_m\}$ a subset  of $X$.  For each $1 \leq j \leq p$, write $S^{(j)}_k$ to denote the $k$-term of the sequence $S^{(j)}$.  This algorithm returns  an $m \times m$ matrix $M$ with entries in $\{0, 1\}$.  Below, we will write $M(s,t)$ to denote the $(s,t)$-entry of $M$.

\begin{algorithmic}[1]
\State For each $1 \leq i \leq m$ and $1 \leq j \leq p$, define a set of positive integers
\[
P_i^{(j)} := \{ k \in \mathbb{Z}^+ : e_i \in S_k^{(j)} \}.
\]
\State For each $1 \leq s \leq m$, set $M(s,s)=0$.
\State For every $1 \leq s \neq t \leq m$:
\If{There is some $j$ such that $P_s^{(j)}, P_t^{(j)}$ are both nonempty and $\max{P_s^{(j)}} < \min{P_t^{(j)}}$, and that for every $k$ such that $P_s^{(k)}, P_t^{(k)}$ are both nonempty   we have $M_{S^{(k)}}(s,t)=1$ as computed by Algorithm \ref{algo:1-v4}}
  \State set $M(s,t)=1$
\Else
  \State set $M(s,t)=0$
\EndIf
\end{algorithmic}
\end{algo}

In later discussions, we will sometimes refer to the output $M$ of Algorithm \ref{algo:1-v4b} as the \emph{common matrix} of the sequences $S^{(1)}, \cdots, S^{(p)}$.

\paragraph[Hasse clustering] In many applications, one might need to  identify the ``common themes'' among the information contained in a collection of sequences.  In order to do this, we can first use  Algorithm \ref{algo:1-v4} to  convert each sequence to a matrix, which in turn  represents the underlying graph of a wiring diagram as explained in \ref{para:algo:1-v4-disc1} and \ref{para:algo:1-v4-disc2}.  We can then use the structure of the category $\mathcal{R}$ to extract the information common to all these individual wiring diagrams.  We will refer to this algorithm as \emph{Hasse clustering}, which we outline in a pseudo-code below.

\begin{algo}[Hasse clustering] \label{algo:3-v1}
Let $X$ be a finite set, $s^{(1)}, ..., s^{(p)}$ a collection of sequences in $X$, and $J=\{e_1, \cdots, e_m\}$ a subset of $X$.  For any real number $t \in [0,100]$ and any positive integer $r$, this algorithm returns a (possibly empty) finite set $\mathcal{C} = \{C_1, \cdots, C_l\}$ where
\begin{itemize}
    \item each $C_i$ is a nonempty set of size at most $r$;
    \item each element of each $C_i$ is an $m \times m$ matrix  with entries in $\{0,1\}$.
\end{itemize}
Below, we will sometimes abuse notation and write $\mathcal{R}(J)$ to denote the set of objects in that category.

\begin{algorithmic}[1]
\State For each $1 \leq i \leq p$, apply Algorithm \ref{algo:1-v4} to the sequence $s^{(i)}$; denote the matrix obtained as the  output  by  $M_i$. 

\State By the discussion in \ref{para:algo:1-v4-disc1} and Lemma \ref{lem:algv1hastrans}, each $M_i$ is the path  matrix of some graph with vertex set $J$; denote this graph by $G_i$. 

\State For every $H$ in $\mathcal{R}(J)$, define the set 
\begin{multline*}
a(H) := \{ G_i : 1 \leq i \leq p, \text{ there is a morphism}  \\
\text{from }G_i \text{ to } H \text{ in the category }\mathcal{R}(J) \}.
\end{multline*}

\State \label{line:definingbarC} Find all the subsets $\overline{C}$ of $\mathcal{R}(J)$ such that $\overline{C}$ has  at most $r$ elements and 
\[
 \frac{\left| \left\{ 1 \leq i \leq p : G_i \in \bigcup\limits_{H \in \overline{C}} a(H) \right\} \right|}{p} \geq t\%.
 \]

\State Construct a directed graph $\mathcal{G}_{cp}$ as follows: the vertices of $\mathcal{G}_{cp}$ are the $\overline{C}$'s found in line \ref{line:definingbarC}.  We define an arrow from vertex $\overline{C}'$ to vertex $\overline{C}''$ if, for every element $G'$ in $\overline{C}'$, there is  some element $G''$ in $\overline{C}''$ such that there exists a morphism from $G'$ to $G''$ in the category $\mathcal{R}(J)$.

\State Let $\overline{C}_1, \cdots, \overline{C}_l$ denote the vertices in $\mathcal{G}_{cp}$ such that, for each $1 \leq i \leq l$, there are no arrows going into $\overline{C}_i$ in the graph $\mathcal{G}_{cp}$.

\State \label{line:convbarCtoC} For each  $1 \leq i \leq l$,  define 
\begin{multline*}
C_i = \{ M: M \text{ is the path matrix of } 
G  \\
\text{ for some }G \in \overline{C}_i\}.
\end{multline*}

\State The output $\mathcal{C}$ of the algorithm is then
\[
\mathcal{C} := \{ C_1, \cdots, C_l\}.
\]
\end{algorithmic}
\end{algo}


In more intuitive terms, we can interpret Algorithm \ref{algo:3-v1} as follows: Given a \emph{threshold} $t$, the algorithm  produces all possible combinations $C$ of (the path matrices of) quasi-skeleton WD graphs such that:
\begin{itemize}
    \item[(i)] Each combination $C$ consists of at most $r$ quasi-skeleton WD graphs.
    \item[(ii)] For each combination $C$,  the quasi-skeleton WD graphs corresponding to  the elements of  $C$ represent the ``themes'' that appear in    at least $t\%$ of the sequences $s^{(i)}$.
    \item[(iii)] Each combination $C$ is ``optimal'' in the sense that it cannot be replaced by another combination $C'$ satisfying (i) and (ii), such that $C'$ contains strictly more information than $C$. 
\end{itemize}

\paragraph[Identifying relevant events]  As noted in \ref{para:catRVHasseDseq}, for a finite set $J$ of size $m$, the number of objects in the category $\mathcal{R}(J)$ is given by sequence A001035 in  the On-Line Encyclopedia of Integer Sequences .  The first 10 terms in this sequence are as shown \cite{OEISposet}:

\begin{center}
\begin{tabular}{ |c|l| } 
\hline
 $m$ & number of objects in $\mathcal{R}(J)$  \\ 
  \hline 
1 & 1\\
2 & 3\\
3 & 19\\
4 & 219\\
5 & 4231\\
6 & 130023\\
7 & 6129859\\
8 & 431723379\\
9 & 44511042511\\
10 & 6611065248783 \\
\hline
\end{tabular}
\end{center}

In our application below, we will analyze data using Algorithm \ref{algo:3-v1} on a personal computer with the following hardware specifications:  
\begin{itemize}
  \item CPU: Intel Core i7-9750H @ 2.60GHz (6 cores, 12 threads)
  \item GPU: NVIDIA GeForce RTX 2060 (6 GB VRAM)
  \item RAM: 16 GB DDR4
  \item Storage: 1 TB NVMe SSD
\end{itemize}
Constrained by the computer power we have access to, we can only apply Algorithm \ref{algo:3-v1} to our data with $m \leq 5$.  As a result, we need to perform basic data mining  to determine the events $e_i$ that are the most relevant before applying  Algorithm \ref{algo:3-v1}.  We focus on the scenario when we have a collection of sequences $s^{(i)}$ where each sequence has been  assigned  a label of either 1 or 0, and we need to find the common themes among those labeled with 1, assuming these themes differentiate those labeled with 1 from those labeled with 0.

\begin{algo} \label{algo:5-v1} Let $X=\{e_1, \cdots, e_n\}$ be a finite set. 
Suppose 
$s^{(1)}, ..., s^{(r)}$ are  sequences in $X$ such that each sequence $s^{(k)}$ has been assigned a label $l(k)$ that equals $1$ or $0$.  This algorithm produces, for each pair $(i,j)$ where $1 \leq i \neq j \leq n$, a score $R_{ij} \in \mathbb{R} \cup \{\infty\}$  that indicates the relevance of the relation `$e_i$ occurs before $e_j$' to a sequence $s^{(k)}$ having label $l(k)$ equal to 1 instead of 0.

We assume that there is at least one sequence $s^{(k)}$ with $l(k)=1$ and at least one sequence $s^{(k')}$ with $l(k')=0$.

\begin{algorithmic}[1]
\State Define positive integers
\begin{itemize}
    \item   $N_W$, the total number of sequences $s^{(k)}$ with $l(k)=1$.
    \item $N_L$, the total number of sequences $s^{(k)}$ with $l(k)=0$.
\end{itemize}

\State For each ordered pair $(i,j)$ where $1 \leq i \neq j \leq n$, define the  integers
\begin{itemize}
    \item $W_{ij}$, the number of sequences $s^{(k)}$ with  $l(k)=1$ and such that  every instance of $e_i$ occurs before every instance of $e_j$.
    \item $L_{ij}$, the number of sequences $s^{(k)}$ with  $l(k)=0$ and such that  every instance of $e_i$ occurs before every instance of $e_j$.
\end{itemize}
and subsequently 
\[
R_{ij} = \begin{cases} \frac{W_{ij}/N_W}{L_{ij}/N_L} &\text{ if $L_{ij} \neq 0$} \\
\infty &\text{ if $L_{ij}=0$} 
\end{cases}
\]

\end{algorithmic}
\end{algo}

The intuition behind Algorithm \ref{algo:5-v1} is, that the higher the score $R_{ij}$, the more likely it is for $e_i$ to always appear before $e_j$ in a sequence $s^{(k)}$ attaining $l(k)=1$.

\section{Application: summarizing behaviors - part 1}\label{sec:application-1}

\paragraph[The problem] Schank and Abelson proposed in the 1970s   that humans  store their knowledge of familiar concepts such as ``going to a restaurant'' or ``riding a bus'' as a \emph{script}, which they defined as ``...a structure that describes appropriate sequences of events in a particular context'', and that children formulate scripts by experiencing multiple instances of the same concept  \cite[Sections 3.1 \& 9.1]{schank1977scripts}.  This theory found further evidence in the work of Nelson and Gruendel, who proposed a more general concept called \emph{general event representation (GER)}, and postulated that GERs provide basic building blocks of cognitive organization  \cite{nelson2013generalized}.  

In this section, we tackle the following problem: If a child can learn what it means to ``buy groceries'' by accompanying their parents to the supermarket on multiple occasions, how can we enable an autonomous system to do the same?  In more formal terms:

\begin{prob}
Given a collection of data that represents multiple observations of the occurrence of a concept, how do we compute a wiring diagram that represents the commonality among the data?
\end{prob}

\paragraph[The data]\label{sec:data-1} To generate data that simulate observations of multiple occurrences of the same concept, we created a simple computer game where the player needs to overcome  obstacles within a maze in order to win the game.  Specifically, the player needs to collect an explosive and a key placed at different locations within the maze, then use the explosive to remove a rock blocking a route leading to a door; the player then wins the game by using the key to open the door.  

\begin{center}
\includegraphics[scale=0.26]{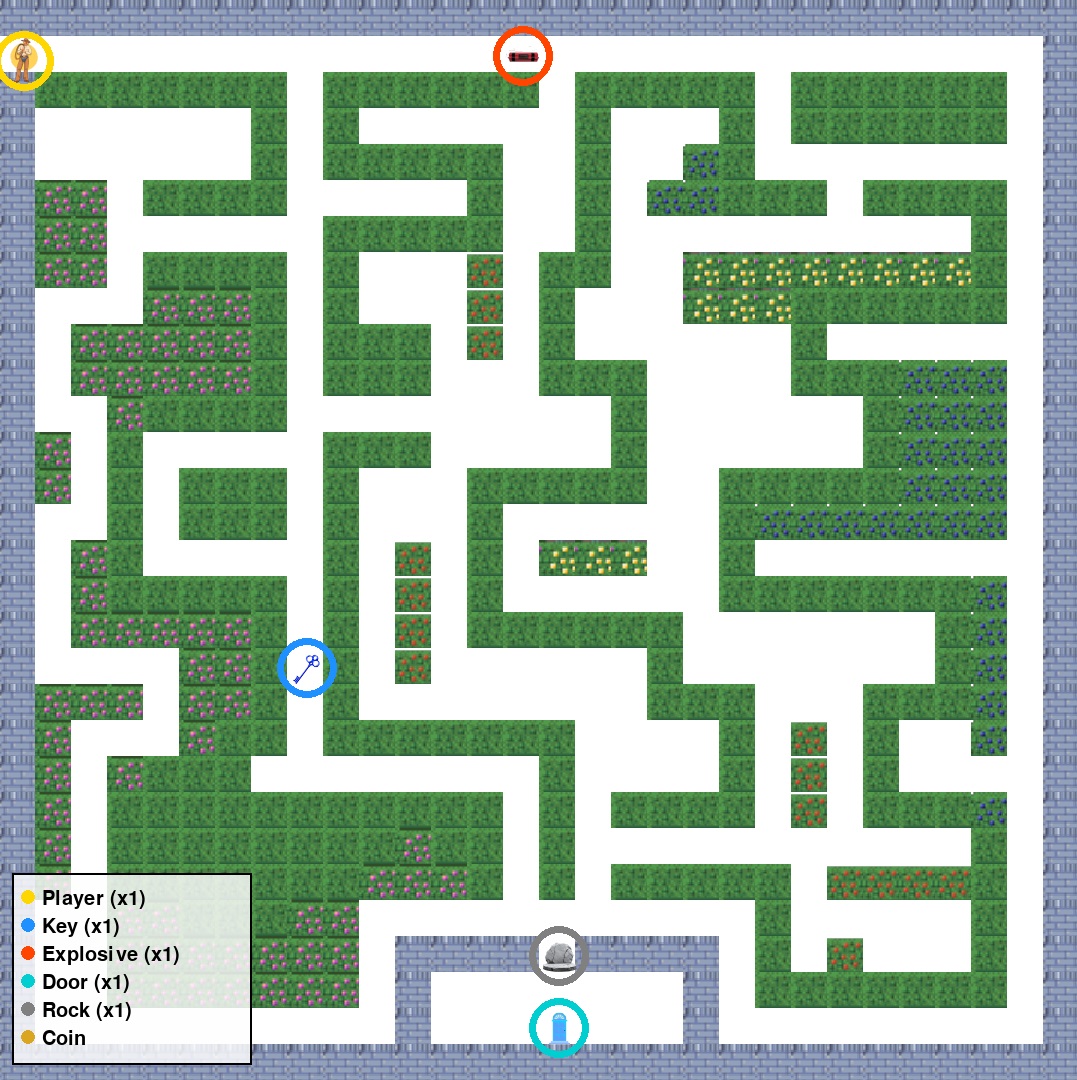}
\end{center}


The collection of all the moves that can possibly be performed by the player is  represented in an olog.  Events of interest, particularly interactions between the player and its environment, are denoted as various $e_i$ as shown  in the table below.

\begin{center}
\begin{tabular}{ |c|l| } 
\hline
 event & corresponding move  \\ 
  \hline 
  $e_1$ & player collects key  \\ 
 $e_2$ &  player collects explosive \\ 
 $e_3$ & key not collected by end of episode  \\ 
 $e_4$ & explosive not collected by end of episode  \\ 
 $e_5$ & player uses explosive on rock  \\ 
 $e_6$ & player uses key on door  \\ 
 $e_7$ & failed use of explosive by player  \\ 
 $e_8$ & failed use of key by player  \\ 
 $e_9$ & player wins episode  \\ 
 $e_{10}$ & player does not win episode  \\ 
 \hline
\end{tabular}
\end{center}

We then used a proximal policy optimization (PPO) algorithm to train a computer agent through reinforcement learning.  The training run produced a total of 284 episodes of the agent playing the game; out of these, the agent won 80 episodes (referred to as \emph{winning episodes}) while failed to complete  the remaining 204 episodes (referred to as \emph{losing episodes}).  For each of these episodes, all the moves performed by the agent were recorded as a sequence.  The sequence was then processed by removing all the data that does not correspond to any of the events $e_1, \cdots, e_{10}$, and then rewritten using the symbols $e_i$.  This resulted in 284 sequences  in the set $X:=\{e_i : 1 \leq i \leq 10\}$.  We denote the sequences as $s^{(k)}$ ($1 \leq k \leq 284$) where $s^{(k)}$ corresponds to a winning episode (resp.\ losing episode) for each $1\leq k \leq 80$ (resp.\ $81 \leq k \leq 284$).





\paragraph[Results \& Analysis] \label{sec:gamev2-resultsanalysis}  To identify the events $e_i$ that are the most relevant to  the agent winning the game, we apply Algorithm \ref{algo:5-v1} to the set $X$ and the sequences $s^{(k)}$ from the previous paragraph together with the function 
\[
l(k) := \begin{cases}
1 &\text{ if $s^{(k)}$ is  from a winning episode} \\
0 &\text{ if $s^{(k)}$ is from a losing episode}
\end{cases}.
\]
The highest value of $R_{ij}$ produced by Algorithm \ref{algo:5-v1} was $\infty$, achieved by the following $(e_i, e_j)$ pairs: 

\begin{gather*}
    (e_1, e_9), (e_1, e_6), (e_6, e_9), (e_8, e_6), \\
    (e_8, e_9), (e_5, e_6), (e_2, e_9), (e_2, e_6), (e_5, e_9)
\end{gather*}
This suggests that the events that are the most relevant to the agent winning the game are 
\[
e_1, e_2, e_5, e_6, e_8, e_9.  
\]
Note that there are 6 of these $e_i$'s  out of 10 possible $e_i$'s, showing that Algorithm \ref{algo:5-v1} helps somewhat in identifying the events relevant to winning the game.


Next, we applied Algorithm \ref{algo:3-v1} to extract information in the form of wiring diagrams from the winning episodes.  Since $e_9$ must appear in any winning episode, we  did not consider it when applying Algorithm \ref{algo:3-v1}.  We also removed  $e_8$ from consideration, since it represents an action by the agent that does not result in any changes to its environment (a ``non-event'').  That is, we applied Algorithm \ref{algo:3-v1} to the set $X$ and the sequences $s^{(k)}$ for $1 \leq k \leq 80$, with the set $J=\{ e_1, e_2, e_5, e_6\}$ and parameters $t=100, r=1$.  We note that the sequences from the winning episodes are all simple sequences in $X$.  

Intuitively,  Algorithm \ref{algo:3-v1} looked for wiring diagrams that were consistent with all the sequences $s^{(k)}$, $1 \leq k \leq 80$ from the winning episodes.  During this process, the objects of the category $\mathcal{R}(J)$ can be visualized as a `heat map' (see Figure \ref{fig:heatmap-1}).

\begin{figure*}[t]
  \centering
  \includegraphics[width=\textwidth]{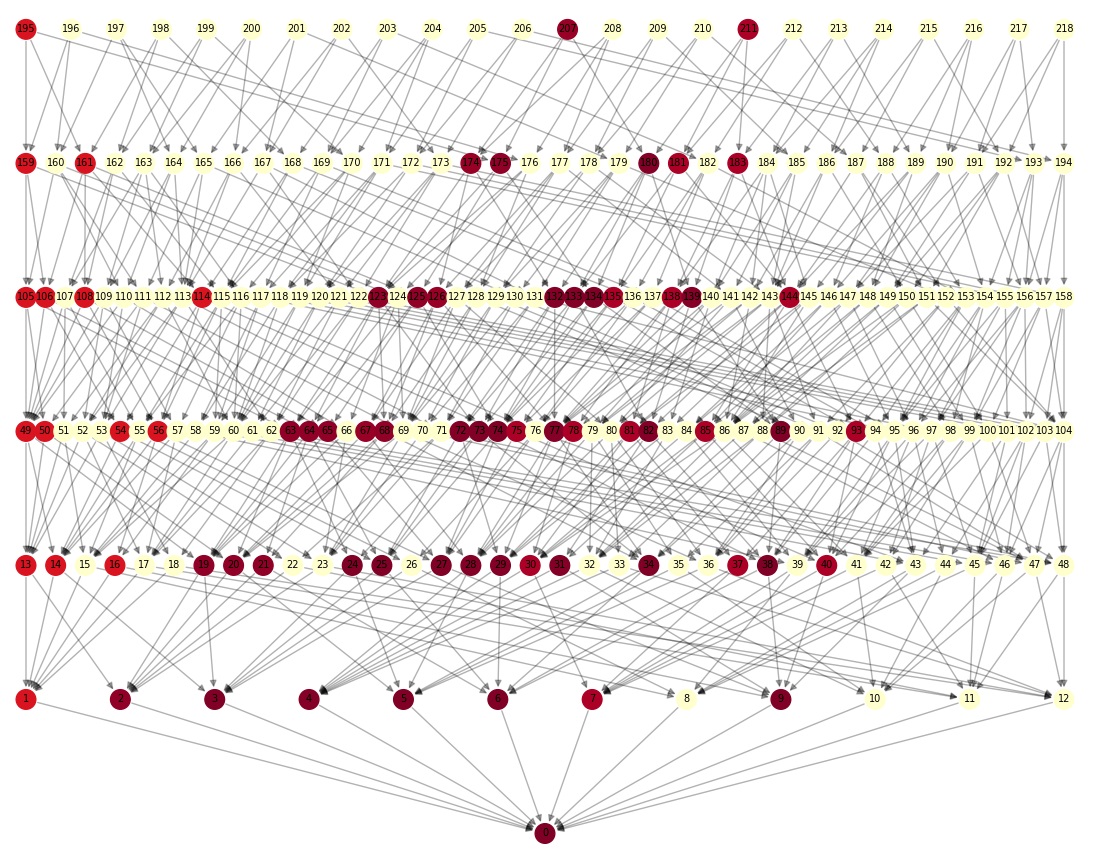}
  \caption{The category $\mathcal{R}(J)$ where $J$ has size 4.  Each node corresponds to a wiring diagram graph.  The higher the color intensity of a node $H$, the higher the value of $a(H)$ in Algorithm \ref{algo:3-v1}.}
  \label{fig:heatmap-1}
\end{figure*}

Algorithm \ref{algo:3-v1} returned an output $\mathcal{C}$ with a single element $C$, where $C$ itself is a singleton set containing the matrix
\begin{equation}\label{eq:output-matrix-1}
\begin{pmatrix}
    0 & 0 & 0 & 1 \\
    0 & 0 & 1 & 1 \\
    0 & 0 & 0 & 1 \\
    0 & 0 & 0 & 0 
\end{pmatrix}
\end{equation}
This matrix is  the path matrix for the graph
\begin{equation}\label{eq:graph-01}
\includegraphics[scale=0.5]{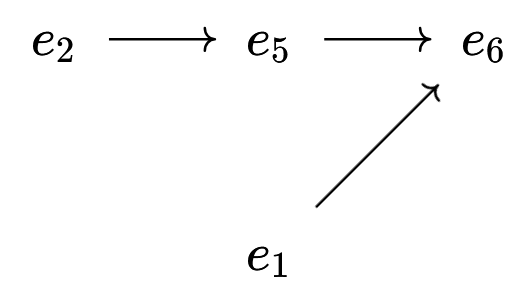}
\end{equation}
which in turn is the result of applying the operation $\wtg$ to the wiring diagram
\begin{equation}\label{eq:WD-01}
\includegraphics[scale=0.5]{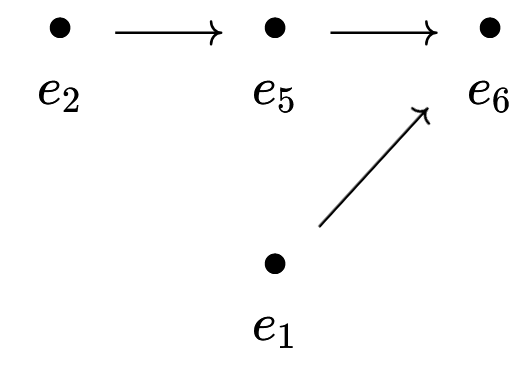}
\end{equation}
The graph \eqref{eq:graph-01} corresponds to node 134 in Figure \ref{fig:heatmap-1}.

The wiring diagram \eqref{eq:WD-01} (or the graph \eqref{eq:graph-01}, or the matrix \eqref{eq:output-matrix-1}) says that based on  the winning episodes, the complete list of before-and-after relations among pairs of elements from $\{e_1, e_2, e_5, e_6\}$ is
\begin{itemize}
    \item $e_1$ always appears before $e_6$.
    \item $e_2$ always appears before $e_5$ as well as $e_6$.
    \item $e_5$ always appears before $e_6$.
\end{itemize}
In plain language, this means that the following actions are always  performed within a winning episode: 
\begin{itemize}
    \item The agent  collects the explosive at some point, and then uses it to destroy the rock, before using the key on the door.  Also, the agent collects the key before using it on the door.
\end{itemize}
These  actions and their order of occurrence,  represented by  the wiring diagram \eqref{eq:WD-01}, \emph{summarize the unique winning strategy for the game}.

\paragraph[Flattenings] In this version of the game,  every   sequence $m_J(s^{(k)})$ (for $1 \leq k \leq 80$) maps to one of the following three possibilities under $\stg$:
\begin{align}
    &\xymatrix{
     e_1 \ar[r] & e_2 \ar[r] & e_5 \ar[r] & e_6 
     } \notag \\
     &\xymatrix{
     e_2 \ar[r] & e_1 \ar[r] & e_5 \ar[r] & e_6 
     } \notag \\
      &\xymatrix{
     e_2 \ar[r] & e_5 \ar[r] & e_1 \ar[r] & e_6 
     } \label{eq:v1-flattenings}
\end{align}
Note that these are precisely the three possible flattenings of the graph \eqref{eq:graph-01}.  Moreover, each of them represents a linear sequence of actions that provides a method for winning the game.  This is an example of Theorem \ref{thm:ON-MR06p2-20-1} parts (i) and (v) in action - the wiring diagram \eqref{eq:WD-01} is, by construction, consistent with each of the sequences  $m_J(s^{(k)})$,  $1 \leq k \leq 80$, and there is a morphism from  \eqref{eq:graph-01} to each of the graphs in \eqref{eq:v1-flattenings}.





\section{Application: summarizing behaviors - part 2}\label{sec:application-2}

\paragraph[The problem] In Section \ref{sec:application-1}, we applied our algorithms to the moves of a player  in a game, and  extracted a wiring diagram that represented the unique winning strategy.  In this section, we modify the game so that there is more than one  winning strategy.   This case  is a proxy for a scenario where multiple themes occur within a set of data, and we would like to identify these distinct themes within the data in the form of wiring diagrams.  (An example of such data could be the clinical data for a group of patients who have either the common flu or COVID - these two respiratory illnesses share similar symptoms, but the orders of occurrence of the symptoms are statistically different \cite{larsen2020modeling}.)

\paragraph[The data] \label{para:gamev3-data} To generate data that simulate observations that contain distinct themes, we modify the game described in Section \ref{sec:data-1} so that there are two different ways for a player to win the game: the player could win by  opening the door as in the earlier version of the game, or by collecting a coin within the maze.   In this version, however, there is a rock obstructing the coin in addition to one obstructing the door.

\begin{center}
\includegraphics[scale=0.25]{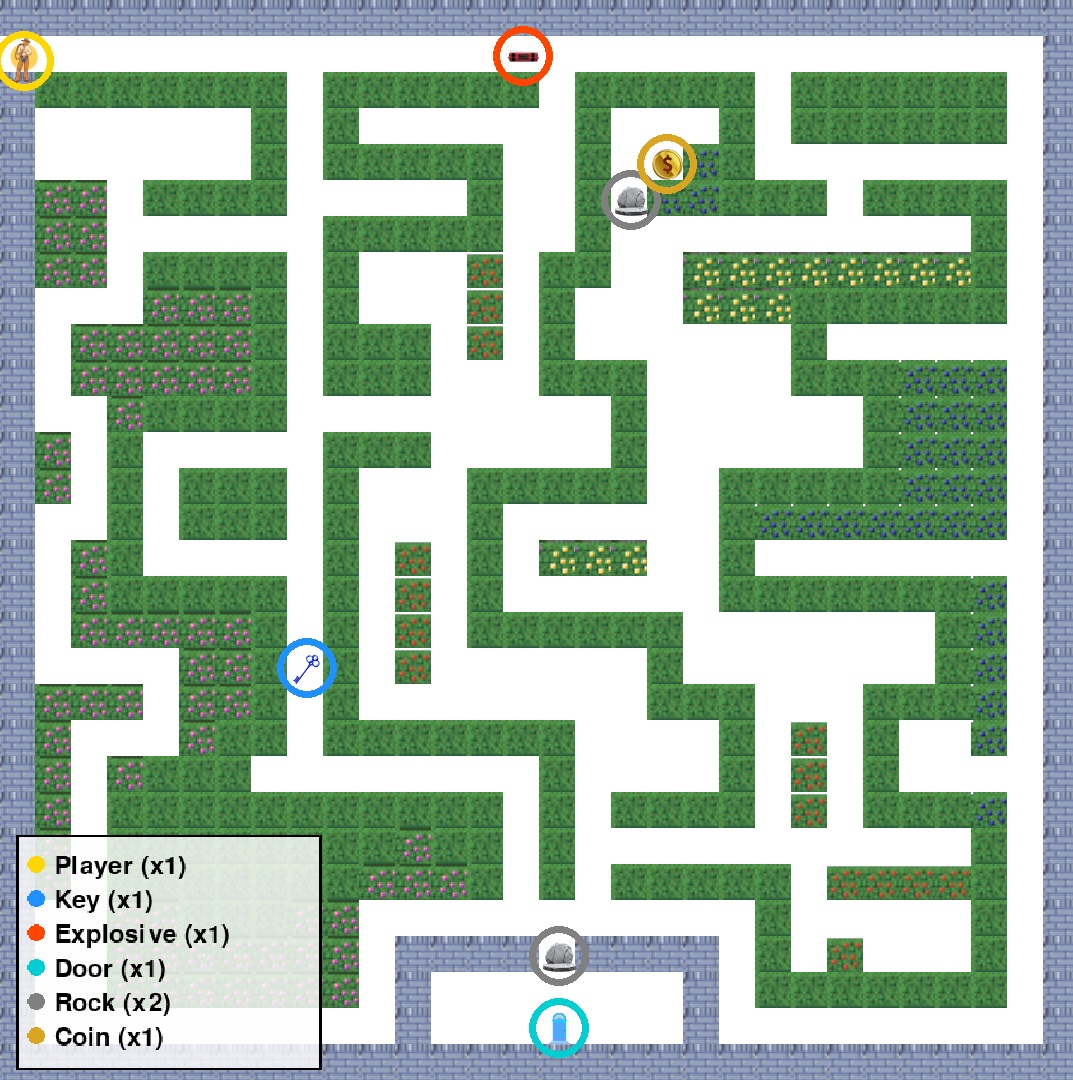}
\end{center}



We again trained a computer agent through reinforcement learning, producing 310 episodes out of which 125 episodes were winning (74 through opening the door, and 51 through collecting the coin) and 185 episodes were losing.  All the moves of the agent were recorded and processed similarly to  the previous version of the game described in Section \ref{sec:data-1}; the only difference was that the definition of  $e_9$ was reworded for clarity and $e_{11}, e_{12}$ were added:

\begin{center}
\begin{tabular}{ |c|l| } 
\hline
 event & corresponding move  \\ 
  \hline 
  $e_1, \cdots, e_8, e_{10}$  & same as in Section  \ref{sec:data-1}   \\ 
 $e_9$ & player wins episode by opening door  \\ 
  $e_{11}$ & player collects coin  \\ 
  $e_{12}$ & player wins by collecting coin \\
 \hline
\end{tabular}
\end{center}

In the end, we obtained  310 sequences $s^{(k)}$ ($1 \leq k \leq 310$) in the set $X:=\{e_i : 1 \leq i \leq 11\}$ where $s^{(k)}$ corresponds to a winning episode (resp.\ losing episode) for $1\leq k \leq 125$ (resp.\ $126 \leq k \leq 310$).  

\paragraph[Results \& Analysis] \label{para:gamev3-results} Using the same function $l$ as in Section \ref{sec:gamev2-resultsanalysis}, we applied Algorithm \ref{algo:5-v1} to the set $X$ and the sequences $s^{(k)}$.  The highest value of $R_{ij}$ produced was $\infty$, achieved by the following pairs $(e_i, e_j)$:

\begin{gather*}
    (e_1, e_{11}), (e_1, e_{12}), (e_1, e_9), (e_1, e_6), (e_2, e_{11}), (e_{11}, e_{12}), \\
    (e_2, e_{12}), (e_5, e_{11}), (e_5, e_{12}), (e_5, e_{6}), (e_3, e_{12}), (e_3, e_{11}), \\
    (e_2, e_9), (e_2, e_{6}), (e_5, e_9), (e_8, e_9), (e_8, e_{12}), (e_8, e_{11}), \\
    (e_8, e_6), (e_6, e_9).
\end{gather*}

This suggests that the events that are the most relevant to the agent winning  are 
\begin{equation}\label{eqn:gamev2-eilistprelim}
e_1, e_2, e_3, e_5, e_6, e_8, e_9,  e_{11}, e_{12}.
\end{equation}
Since there are 9 of these $e_i$'s out of the 12 possible $e_i$'s, Algorithm \ref{algo:5-v1} is less effective in identifying events relevant to winning the game compared to the scenario in Section \ref{sec:application-1}.

Since the  compute power we have access to only allows us to run Algorithm \ref{algo:3-v1} for $m \leq 5$, we need to reduce the list of relevant $e_i$ from \eqref{eqn:gamev2-eilistprelim} to a subset of size 5 or less.  Similar to what we did  in Section \ref{sec:gamev2-resultsanalysis}, we omitted $e_9$ and $e_{12}$ because they always follow $e_6, e_{11}$, respectively.  We also omitted $e_3, e_8$ since they do not represent changes in the game environment.  As a result, we focused on the set of events $J=\{ e_1, e_2, e_5, e_6, e_{11}\}$.  We then applied Algorithm \ref{algo:3-v1}  to $X$, $J$, and $s^{(k)}$ for $1 \leq k \leq 125$, using the parameters $t=100$ and $r=2$.  That is, we looked for combinations  of wiring diagrams of size at most 2, such that every wiring diagram in $C$ is consistent with all the sequences $s^{(k)}$ from winning episodes.  As in the previous section, the sequences from the winning episodes were all simple sequences in $X$.

Algorithm \ref{algo:3-v1} returned a set $\mathcal{C}$ with a single element $C$, which itself has two elements:
\begin{gather}
    \begin{pmatrix}
0 & 0 & 0 & 0 & 0 \\
0 & 0 & 1 & 0 & 1 \\
0 & 0 & 0 & 0 & 1 \\
0 & 0 & 0 & 0 & 0 \\
0 & 0 & 0 & 0 & 0
\end{pmatrix} \label{eq:gamev3-mtx1}  \\
\begin{pmatrix}
0 & 0 & 0 & 1 & 0 \\
0 & 0 & 1 & 1 & 0 \\
0 & 0 & 0 & 1 & 0 \\
0 & 0 & 0 & 0 & 0 \\
0 & 0 & 0 & 0 & 0
\end{pmatrix}. \label{eq:gamev3-mtx2}
\end{gather}
These are the path matrices for the following respective graphs (the corresponding wiring diagrams are clear):
\begin{gather}
\includegraphics[scale=0.58]{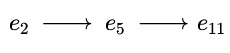}\label{eq:gamev3-WD-1}\\
\includegraphics[scale=0.5]{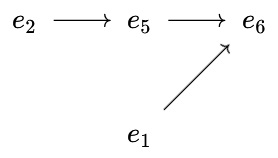}\label{eq:gamev3-WD-2}
\end{gather}
This means that every winning episode is consistent with either of these two strategies for the game:
\begin{itemize}
    \item Strategy 1 (represented by \eqref{eq:gamev3-WD-1}): The player collects the explosive, then uses it to destroy a rock (necessarily the one obstructing the coin), and then collects the coin.
    \item Strategy 2 (represented by \eqref{eq:gamev3-WD-2}): The player collects the explosive, then uses it to destroy a rock (necessarily the one obstructing the door), and then uses the key on the door.  Also, the player  collects the key  before using it on the door.
\end{itemize}
It is easy to see that  these are indeed \emph{the} two winning strategies for the game.

\paragraph[Omitting extraneous information]\label{para:gamev3-extra} In this version of the game,  every   sequence $m_J(s^{(k)})$ (for $1 \leq k \leq 125$) maps to one of the following seven possibilities under  $\stg$:
\begin{itemize}
    \item the graph \eqref{eq:gamev3-WD-1};
    \item the three flattenings of \eqref{eq:gamev3-WD-2};
    \item the graphs 
    \begin{align}
        &\xymatrix{
        e_1 \ar[r] & e_2 \ar[r] & e_5 \ar[r] & e_{11} 
        } \label{eq:gamev3-WD-extra-1}\\
        &\xymatrix{
        e_2 \ar[r] & e_1 \ar[r] & e_5 \ar[r] & e_{11} 
        } \label{eq:gamev3-WD-extra-2} \\
            &\xymatrix{
        e_2 \ar[r] & e_5 \ar[r] & e_1 \ar[r] & e_{11} 
        } \label{eq:gamev3-WD-extra-3}
    \end{align}
\end{itemize}
Each of these seven graphs represents a linear sequence of actions that ensures the player would win the game.  Note, however, that performing the action $e_1$ is not necessary for winning the game when following the recipes in  \eqref{eq:gamev3-WD-extra-1}, \eqref{eq:gamev3-WD-extra-2}, and \eqref{eq:gamev3-WD-extra-3}.  There are morphisms from each of these three graphs to \eqref{eq:gamev3-WD-1} in the category $\mathcal{R}$, and we can informally think of  Algorithm \ref{algo:3-v1} as having the ability to ``remove'' the  extraneous step  $e_1$  from \eqref{eq:gamev3-WD-extra-1}, \eqref{eq:gamev3-WD-extra-2}, and \eqref{eq:gamev3-WD-extra-3}.

\section{Comparison with standard clustering algorithms}\label{sec:comparison}

In this section, we study the use of two standard clustering algorithms - DBSCAN and (agglomerative) hierarchical clustering - in conjunction with Algorithm \ref{algo:1-v4b}.  We applied these methods to  the data from our autonomous agent playing the game.  Since data from winning episodes in version one of the game only gave three possible sequences $m_J(s^{(k)})$ (see Section \ref{sec:application-1}) whereas version two gave seven possible sequences $m_J(s^{(k)})$ (see Section \ref{sec:application-2}), we discuss only the more interesting case of version two.  Each of the two standard clustering algorithms sorts the 125 winning episodes into clusters, and Algorithm \ref{algo:1-v4b} converts each cluster to the path matrix of a wiring diagram graph.


\paragraph[DBSCAN plus Algorithm \ref{algo:1-v4b}]\label{sec:comparison-DBSCAN} In this method, we applied Algorithm \ref{algo:1-v4} to each of the sequences $s^{(k)}$, $1 \leq k \leq 125$ with the same inputs $X$ and $J$ as in Section  \ref{sec:application-2}.  Let $M_k$ denote the resulting matrix from the sequence $s^{(k)}$.  We then applied DBSCAN to the set of matrices $\{M_k\}_{1 \leq k \leq 125}$ using the $L^1$-norm, \texttt{min\_samples}$=1$, and various values of $\epsilon$.  The number of clusters produced is shown below:

\begin{center}
\begin{tabular}{ |c|c| } 
\hline
 $\epsilon$ & number of clusters  \\ 
  \hline 
   0 & 7  \\ 
 1 &  7 \\ 
   2  & 3  \\ 
   3   & 2  \\ 
   4   &  2 \\ 
   5   &  2 \\ 
    $\geq 6$   &  1  \\ 
 \hline
\end{tabular}
\end{center}
For each cluster obtained using DBSCAN, we then computed the common matrix of all the sequences in the cluster using  Algorithm \ref{algo:1-v4b}. 

When $\epsilon =2$, we obtain three clusters with sizes 74, 40, and 11, with common matrices \eqref{eq:gamev3-mtx2}, \eqref{eq:gamev3-mtx1}, and 
\begin{equation}\label{eq:gamev3-DBSCAN-extracluster1}
    \begin{pmatrix}
    0 & 0 & 0 & 0 & 1 \\
    0 & 0 & 1 & 0 & 1 \\
    0 & 0 & 0 & 0 & 1 \\
    0 & 0 & 0 & 0 & 0 \\
    0 & 0 & 0 & 0 & 0 
    \end{pmatrix},
\end{equation}
respectively.  When  $\epsilon = 3, 4, 5$, the common  matrices of the two clusters are always \eqref{eq:gamev3-mtx2} and \eqref{eq:gamev3-DBSCAN-extracluster1}.  Here, \eqref{eq:gamev3-DBSCAN-extracluster1} is the path matrix  of   
\begin{equation} \label{eq:gamev3-WD-extra-common}
    \xymatrix{
    e_2 \ar[r] & e_5 \ar[r] & e_{11} \\
    & e_1 \ar[ur] &
    }
\end{equation}
We can certainly think of \eqref{eq:gamev3-WD-extra-common} as a  ``summary'' of \eqref{eq:gamev3-WD-extra-1}, \eqref{eq:gamev3-WD-extra-2}, and \eqref{eq:gamev3-WD-extra-3}; however, as we saw in \ref{para:gamev3-extra}, these are not the only winning strategies.

The above results show, that even though DBSCAN produced wiring diagrams that correspond to  strategies sufficient for  winning  the game, these strategies may not be the `necessary' conditions for winning, and may not account for all the possible winning strategies.



\paragraph[Hierarchical clustering plus Algorithm \ref{algo:1-v4b}] \label{sec:comparison-hier}     We computed matrices $M_k$, $1 \leq k \leq 125$ in the same manner as in Section \ref{sec:comparison-DBSCAN}.  We then applied agglomerative hierarchical clustering to these matrices using the $L^1$-norm and  the average linkage criterion.  The resulting dendrogram  is shown in Figure \ref{fig:hierclustering}, and  the  number of clusters obtained for varying distance thresholds are shown below:

\begin{figure*}[t]
  \centering
  \includegraphics[width=\textwidth]{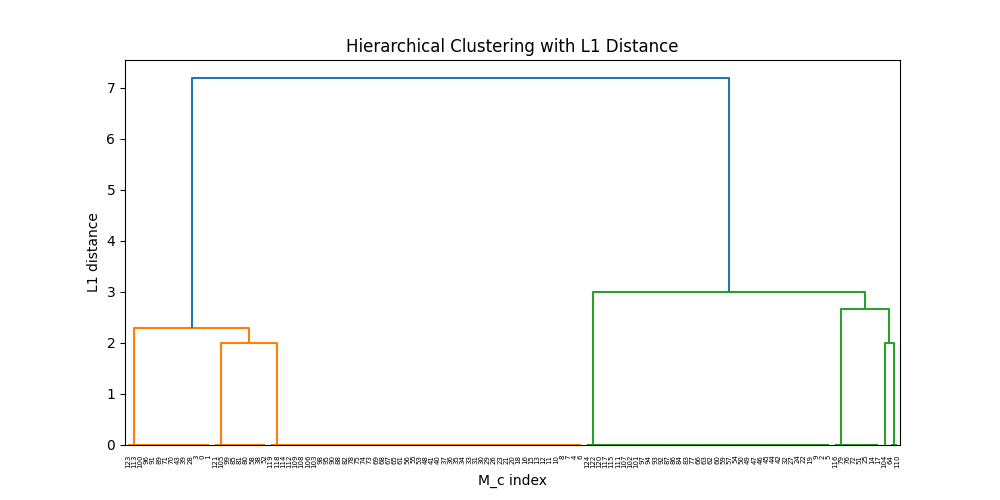}
  \caption{Dendrogram for hierarchical clustering with $L^1$-norm, applied to winning episodes of version two of the game.}
  \label{fig:hierclustering}
\end{figure*}

\begin{center}
\begin{tabular}{ |c|c| } 
\hline
 threshold & number of clusters  \\ 
  \hline 
   0 & 7  \\ 
 1 &  7 \\ 
   2  & 5  \\ 
   3   & 2  \\ 
   4   &  2 \\ 
   5   &  2 \\ 
6    &  2  \\ 
7 &  2  \\ 
$\geq 8$    &  1  \\ 
 \hline
\end{tabular}
\end{center}

For distance threshold equal to $3, 4, \cdots, 7$, the common matrices of the two clusters are path matrices of \eqref{eq:gamev3-WD-2} and \eqref{eq:gamev3-WD-extra-common}.  That is, as in the case of `DBSCAN plus plus Algorithm \ref{algo:1-v4b}', `agglomerative hierarchical clustering plus  Algorithm \ref{algo:1-v4b}' identifies sufficient, but not necessary conditions for winning the game.


\section{Comparison with standard clustering algorithms - with corrupted data}\label{sec:comparison-corruption}

When humans try to formulate an abstract concept out of everyday experiences, the `data' from these experiences can  contain imperfections due to miscommunication or distortion of memory.  In this section, we consider how corruption of data may affect the performances of:
\begin{itemize}
    \item Hasse clustering (Algorithm \ref{algo:3-v1});
    \item DBSCAN plus Algorithm \ref{algo:1-v4b};
    \item agglomerative Hierarchical clustering plus Algorithm \ref{algo:1-v4b}.
\end{itemize}




To test the robustness of the above three schemes of extracting wiring diagrams from data, we corrupted 10\% of the 125 sequences $s^{(k)}$ ($1 \leq k \leq 125$) from winning episodes in the second version of the game (see Section \ref{para:gamev3-data}).  The methods of corruption include: swapping terms in a sequence, deleting terms in a sequence, or inserting new terms in a sequence.  

We then applied Algorithm \ref{algo:1-v4} to the same $X$ and $J$ as in Section \ref{sec:application-2} and the sequences $s^{(k)}, 1 \leq k \leq 125$.  This resulted in 1 zero matrix, and 124 nonzero matrices; we retain the 124 sequences that yield nonzero matrices and relabel them as  $s^{(k)}, 1 \leq k \leq 124$. 





\paragraph[Hasse clustering] We applied Hasse clustering to  $X, J$, and $s^{(k)}, 1 \leq k \leq 124$ as above, with parameters $t=90$ and $r=2$.   We obtained  the same output as in Section \ref{para:gamev3-results}, namely one set containing two matrices \eqref{eq:gamev3-mtx1} and \eqref{eq:gamev3-mtx2} which correspond to the two distinct winning strategies for the game.

\paragraph[DBSCAN plus Algorithm \ref{algo:1-v4b}]  \label{sec:robustness-DBSCAN} We proceeded as in  \ref{sec:comparison-DBSCAN}: First, we applied Algorithm \ref{algo:1-v4} to each $s^{(k)}, 1 \leq k \leq 124$ with $X, J$ as above; let $M_k$ denote the matrix produced by the algorithm.  Then, we applied DBSCAN to the set of matrices $\{M_k\}_{1 \leq k \leq 124}$ using the $L^1$-norm and  \texttt{min\_samples}$=1$.


For different values of $\epsilon$, the  number of clusters obtained and the cluster sizes are shown below.

\begin{center}
\begin{tabular}{ |c|c|p{4.1cm}| } 
\hline
 $\epsilon$ & number of clusters & cluster sizes \\ 
  \hline 
   1 & 17 & 12, 37, 2, 46, 6, 1, 1, 9, 1, 1, 2, 1, 1, 1, 1, 1, 1 \\ 
 2 &  8 & 68, 42, 9, 1, 1, 1, 1, 1\\ 
   3  & 2 & 123, 1 \\ 
   $\geq 4$   & 1 & 124 \\ 
 \hline
\end{tabular}
\end{center}

In the case when two clusters are produced, i.e.\ when $\epsilon =3$, the common matrices are:
\begin{equation}\label{eq:robust-v3-DBSCAN-epis3-1}
    \begin{pmatrix}
        0 & 0 & 0 & 0 & 1 \\
        0 & 0 & 0 & 0 & 0 \\
         0 & 0 & 0 & 0 & 0 \\
          0 & 0 & 0 & 0 & 0 \\
          0 & 0 & 0 & 1 & 0 
    \end{pmatrix}
\end{equation}
for the cluster of size 123, and 
\begin{equation}\label{eq:robust-v3-DBSCAN-epis3-2}
    \begin{pmatrix}
        0 & 0 & 1 & 0 & 0 \\
        0 & 0 & 0 & 0 & 0 \\
         0 & 0 & 0 & 0 & 0 \\
          0 & 0 & 0 & 0 & 0 \\
          1 & 0 & 1 & 0 & 0 
    \end{pmatrix}
\end{equation}
for the cluster of size 1.  Since the latter cluster has size 1, its common matrix is the same as the unique matrix in that cluster.

Note that \eqref{eq:robust-v3-DBSCAN-epis3-1} is \emph{not the path matrix of any wiring diagram graph}; in particular, it is the adjacency matrix of the linear graph
\begin{equation}\label{eq:robust-gamev3-graph2}
\xymatrix{
e_1 \ar[r] & e_{11} \ar[r] & e_6
}
\end{equation}
but not its path matrix.  On the other hand, \eqref{eq:robust-v3-DBSCAN-epis3-2} is the path matrix of the graph
\begin{equation}\label{eq:robust-gamev3-graph1}
\xymatrix{
e_{11} \ar[r] & e_{1} \ar[r] & e_5
}
\end{equation}
Neither \eqref{eq:robust-gamev3-graph1} nor \eqref{eq:robust-gamev3-graph2}, however, represents a sufficient strategy for winning the game.

\paragraph[Hierarchical clustering plus Algorithm \ref{algo:1-v4b}]   We proceeded as in  \ref{sec:comparison-hier}: First, we computed matrices $M_k$ via Algorithm \ref{algo:1-v4} using the same $X, J, s^{(k)}$ as in \ref{sec:robustness-DBSCAN}.  Then, we   applied agglomerative hierarchical clustering to these matrices using the $L^1$-norm and  the average linkage criterion.  For different thresholds, the  number of clusters obtained and the cluster sizes are shown below.

\begin{center}
\begin{tabular}{ |c|c|p{3.2cm}| } 
\hline
 threshold & number of clusters & cluster sizes \\ 
  \hline 
   0, 1 & 17 & 12, 1, 9, 46, 1, 1, 1, 1, 2, 6, 37, 1, 2, 1, 1, 1, 1 \\ 
 2 &  12 & 13, 55, 1, 1, 1, 1, 8, 37, 1, 3, 2, 1 \\ 
   3  & 9 & 68, 1, 2, 1, 8, 38, 3, 2, 1 \\ 
   4   & 7 &  68, 1, 2, 1, 46, 5, 1 \\ 
   5 & 3 & 69, 3, 52\\
   6, 7 & 2 & 69, 55\\
   $\geq 8$ & 1 & 124 \\
 \hline
\end{tabular}
\end{center}
The only thresholds that produce two clusters as in Hasse clustering are 6 and 7.  In these cases, the  common matrices are:
\begin{equation}\label{eq:robust-v3-hier-epis3-1}
    \begin{pmatrix}
        0 & 0 & 0 & 0 & 0 \\
        0 & 0 & 0 & 1 & 0 \\
         0 & 0 & 0 & 1 & 0 \\
          0 & 0 & 0 & 0 & 0 \\
          0 & 0 & 0 & 0 & 0 
    \end{pmatrix}
\end{equation}
for the cluster with 69 matrices, and
\begin{equation}\label{eq:robust-v3-hier-epis3-2}
    \begin{pmatrix}
        0 & 0 & 0 & 0 & 0 \\
        0 & 0 & 0 & 0 & 0 \\
         0 & 0 & 0 & 0 & 0 \\
          0 & 0 & 0 & 0 & 0 \\
          0 & 0 & 0 & 1 & 0 
    \end{pmatrix}
\end{equation}
for the cluster with 55 matrices.  These are  the path matrices of 
\begin{equation}\label{eq:robust-gamev3-graph3}
\xymatrix{
e_2 \ar[r]  & e_6 \\
e_5 \ar[ur] &
}
\end{equation}
and
\begin{equation}\label{eq:robust-gamev3-graph4}
\xymatrix{
 e_{11} \ar[r] & e_6
},
\end{equation}
respectively.  Clearly, neither of them provides a sufficient winning strategy.

\section{Conclusion}\label{sec:article-conclusion}

\paragraph[What we did] In this article, we proved mathematical results on the enumeration of wiring diagrams as well as their categorical properties.  Using these theoretical results, we designed algorithms for extracting information - in the form of wiring diagrams - from sequential data.  We then tested our algorithms on the data of observed behavior of an autonomous agent playing a computer game.  We also compared the output of our algorithm with that of two other approaches based on standard clustering algorithms such as DBSCAN and agglomerative hierarchical clustering.  We then performed the comparison again using corrupted data.

\paragraph[What we found] We proved that quasi-skeleton wiring diagrams are in 1-1 correspondence with Hasse diagrams, the enumeration of which is already well-known.  We also proved a universal property for wiring diagrams: in order to compare a wiring diagram to sequential data, we must consider `flattenings' of the wiring diagram.  

The main algorithm we designed, Hasse clustering (Algorithm \ref{algo:5-v1}), was successful in extracting abstract concepts from time series: when applied to sequential data representing behaviors of an autonomous agent playing a computer game, our algorithm correctly identified the winning strategies.  

We found Hasse clustering to be more effective than the two approaches based on DBSCAN and agglomerative hierarchical clustering. When applied to corrupted data, Hasse clustering was more resilient than both DBSCAN and hierarchical clustering.  

\paragraph[Future directions] At present, Hasse clustering can only run on a personal computer when the number of relevant events $m$ is at most 5.  If one is to run Hasse clustering without access to  higher compute power, it will be necessary to modify  Algorithm \ref{algo:5-v1} so that it can also be run when $m$ is greater than 5.

Since Hasse clustering can be applied to any time series, one could  also apply it to data in various other contexts such as finance, medicine, education, communication, etc.,  to see if new insights can be uncovered that are otherwise unavailable using  existing algorithms.

\paragraph[Data and code availability] All code and data used in this study are publicly available in our GitHub repository  \cite{nima-git01}. The repository is organized to mirror the pipeline described in this paper and includes the full set of clustering scripts used in our analysis. To ensure reproducibility, we include both the raw data (compressed when necessary) and the processed outputs. The GitHub repository provides a complete and transparent implementation of our framework: from custom environment design, reinforcement learning training, and pre-processing of symbolic event sequences to dependency matrix construction and clustering. This makes it possible for researchers to both reproduce the results presented here and extend the approach to new environments or alternative analysis techniques. 

To complement the repository, we provide two citable technical reports on Zenodo that detail version-specific workflows: \cite{nima-report-v2} documents the single-solution environment (version one of the game in this paper - see Section \ref{sec:application-1}), including the pipeline, dependency extraction, Hasse-based clustering, and robustness studies, while \cite{nima-report-v3} covers the two-solution environment with the same components (version two of the game - see Section \ref{sec:application-2}).




\section{Acknowledgments}  This material is based upon work supported by the Air Force Office of Scientific Research under award 
number FA9550-24-1-0268,  and by a DARPA Young Faculty Award  D21AP10109-02. We also thank  Csaba Toth and Lance Cruz for helpful discussions.

\end{multicols}

\bibliography{refsMR2}{}
\bibliographystyle{plain}

\end{document}